\documentclass[sigconf, nonacm]{acmart}

\AtBeginDocument{%
  \providecommand\BibTeX{{%
    \normalfont B\kern-0.5em{\scshape i\kern-0.25em b}\kern-0.8em\TeX}}}

\copyrightyear{2021}
\acmYear{2021}
\setcopyright{acmcopyright}
\acmConference[KDD '21] {Proceedings of the 27th ACM SIGKDD Conference on Knowledge Discovery and Data Mining}{August 14--18, 2021}{Virtual Event, Singapore.}
\acmBooktitle{Proceedings of the 27th ACM SIGKDD Conference on Knowledge Discovery and Data Mining (KDD '21), August 14--18, 2021, Virtual Event, Singapore}
\acmPrice{15.00}
\acmISBN{978-1-4503-8332-5/21/08}
\acmDOI{10.1145/3447548.3467251}
\usepackage{color}
\newtheorem{definition}{Definition}
\newtheorem{theorem}{Theorem}
\newtheorem{proposition}{Proposition}
\usepackage[ruled,vlined]{algorithm2e}

\usepackage{graphicx}
\usepackage{subfigure}
\usepackage{enumitem}
\usepackage{multirow}
\usepackage{amsthm,amsmath}
\usepackage{multicol}
\renewcommand\footnotetextcopyrightpermission[1]{}
\newcommand{\sen}[1]{{\color{black}{#1}}}
\newcommand{\pan}[1]{{\color{black}{#1}}}

\newcommand\blfootnote[1]{%
  \begingroup
  \renewcommand\thefootnote{}\footnote{#1}%
  \addtocounter{footnote}{-1}%
  \endgroup
}

\begin{document}
	\fancyhead{}

\title{Towards Model-Agnostic Post-Hoc Adjustment for Balancing Ranking Fairness and Algorithm Utility}



\author{Sen Cui$^{1*}$, Weishen Pan$^{1*}$, Changshui Zhang$^{1}$, Fei Wang$^2$}
\affiliation{%
  \institution{$^1$Institute for Artificial Intelligence, Tsinghua University (THUAI), State Key Lab of Intelligent Technologies and Systems,Beijing National Research Center for Information Science and Technology (BNRist) \\ Department of Automation, Tsinghua University, Beijing, P.R.China\\$^2$Department of Population Health Sciences, Weill Cornell Medicine, USA}
  \country{}
}
\email{{cuis19, pws15}@mails.tsinghua.edu.cn, zcs@mail.tsinghua.edu.cn,few2001@med.cornell.edu}

\renewcommand{\authors}{Sen Cui, Weishen Pan, Changshui Zhang, Fei Wang}

\begin{abstract}
Bipartite ranking, which aims to learn a scoring function that ranks positive individuals higher than negative ones from labeled data, is widely adopted in various applications where sample prioritization is needed. Recently, there have been rising concerns on whether the learned scoring function can cause systematic disparity across different protected groups defined by sensitive attributes. While there could be trade-off between fairness and performance, in this paper we propose a model agnostic post-processing framework for balancing them in the bipartite ranking scenario. Specifically, we maximize a weighted sum of the utility and fairness by directly adjusting the relative ordering of samples across groups. By formulating this problem as the identification of an optimal warping path across different protected groups, we propose a non-parametric method to search for such an optimal path through a dynamic programming process. Our method is compatible with various classification models and applicable to a variety of ranking fairness metrics. Comprehensive experiments on a suite of benchmark data sets and two real-world patient electronic health record repositories show that our method can achieve a great balance between the algorithm utility and ranking fairness. Furthermore, we experimentally verify the robustness of our method when faced with the fewer training samples and the difference between training and testing ranking score distributions. \blfootnote{$*$ Equal contributions from both authors.}
\end{abstract}



\maketitle

\section{Introduction}
Machine learning algorithms have been widely applied in a variety of real-world applications including the high-stakes scenarios such as loan approvals, criminal justice, healthcare, etc. An increasing concern is whether these algorithms make fair decisions in these cases. For example, ProPublica reported that an algorithm used across the US for predicting a defendant’s risk of future crime produced higher scores to African-Americans than Caucasians on average~\cite{angwin2016machine}. This stimulates lots of research on improving the fairness of the decisions made by machine learning algorithms.

Existing works on fairness in machine learning have mostly focused on the disparate impacts of binary decisions informed by algorithms with respect to different groups formed from the protected variables (e.g., gender or race). Demographic parity requires the classification results to be independent of the group memberships. Equalized odds~\cite{hardt2016equality} seeks for equal false positive and negative rates across different groups. Accuracy parity~\cite{zafar2017fairness} needs equalized error rates across different groups.

Another scenario that frequently involves computational algorithms is ranking. For example, Model for End-stage Liver Disease (MELD) score, which is derived from a simple linear model from several features, has been used for prioritizing candidates who need liver transplantation~\cite{wiesner2003model}. Studies have found that women were less likely than men to receive a liver transplant within 3 years with the MELD score~\cite{moylan2008disparities}. To quantify ranking fairness, Kallus {\em et al.}~\cite{kallus2019fairness} proposed \emph{xAUC}, which measures the probability of positive examples of one group being ranked above negative examples of another group. Beutel {\em et al.}~\cite{beutel2019fairness} proposed a similar definition \emph{pairwise
	ranking fairness} (PRF), which requires equal probabilities for positive instances from each group ranked above all negative instances.

To address the potential disparity induced from risk scores, Kallus {\em et al.}~\cite{kallus2019fairness} proposed a post-processing approach that adjusts the risk scores of {the instances in the disadvantaged group} with a parameterized monotonically increasing function. This method is model agnostic and aims to achieve equal xAUC, but it does not consider algorithm utility (i.e., AUC) explicitly. Beutel {\em et al.}~\cite{beutel2019fairness} studied the balance between algorithm utility and ranking fairness and proposed an optimization framework by minimizing an objective including the classification loss and a regularization term evaluating the absolute correlation between the group membership and pairwise residual predictions. Though this method considers both utility and fairness, is model-dependent and does not directly optimize PRF disparity but an approximated proxy.

In this paper, we develop a model agnostic post-processing framework, \texttt{xOrder}, to achieve ranking fairness and maintain the algorithm utility. Specifically, we show that both algorithm utility and ranking fairness are essentially determined by the ordering of the instances involved. \texttt{xOrder} makes direct adjustments of the cross-group instance ordering (while existing post-processing algorithms mostly aimed at adjusting the ranking scores to optimize the ordering). The optimal adjustments can be obtained through a dynamic programming procedure of minimizing an objective comprising a weighted sum of algorithm utility loss and ranking disparity. We theoretically analyze our method in two cases. If we focus on maximizing the utility, \texttt{xOrder} achieves a global optimal solution. While we care only about minimizing the disparity, it can have a relatively low bound of ranking disparity. The learned ordering adjustment can be easily transferred to the test data through linear interpolation.

We evaluate \texttt{xOrder} empirically on four popular benchmark data sets for studying algorithm fairness and two real-world electronic health record data repositories. The results show \texttt{xOrder} can achieve low ranking disparities on all data sets while at the same time maintaining good algorithm utilities. In addition, we compare the performance of \texttt{xOrder} with another post-processing algorithm when faced with the difference between training and test distributions. From the results, we find our algorithm can achieve robust performance when training and test ranking score distributions are significantly different. The source codes of \texttt{xOrder} are made publicly available at \url{https://github.com/cuis15/xorder}.


\section{Related Works}
Algorithm fairness is defined as the disparities in the decisions made across groups formed by protected variables, such as gender and race. Many previous works on this topic focused on binary decision settings. Researchers have used different proxies as fairness measures which are required to be the same across different groups for achieving fairness. Examples of such proxies include the proportion of examples classified as positive~\cite{calders2009building,calders2010three}, as well as the prediction performance metrics such as true/false positive rates and error rates~\cite{dixon2018measuring,feldman2015certifying,hardt2016equality,zafar2017fairness,kallus2018residual}. A related concept that is worthy of mentioning here is calibration~\cite{lichtenstein1981calibration}. A model with risk score $\operatorname{S}$ on input $\operatorname{X}$ to generate output $\operatorname{Y}$ is considered calibrated by group if for $\forall s\in[0,1]$, we have $\operatorname{Pr}(\operatorname{Y} = 1| \operatorname{S} = s, \operatorname{A} = a) = \operatorname{Pr}(\operatorname{Y} = 1| \operatorname{S} = s, \operatorname{A} = b)$ where $\mathrm{A}$ is the group variable~\cite{chouldechova2017fair}. Recent studies have shown that it is impossible to satisfy both error rate fairness and calibration simultaneously when the prevalence of positive instances are different across groups~\cite{kleinberg2016inherent,chouldechova2017fair}. Plenty of approaches have been proposed to achieve fairness in binary classification settings. One type of method is to train a classifier without any adjustments and then post-process the prediction scores by setting different thresholds for different groups~\cite{hardt2016equality}. Other methods have been developed for optimization of fairness metrics during the model training process through adversarial learning~\cite{zemel2013learning,louizos2015variational,beutel2017data,madras2018learning,zhang2018mitigating} or regularization~\cite{kamishima2011fairness,zafar2015fairness,beutel2019putting}.

Ranking fairness is an important issue in applications where the decisions are made by algorithm produced ranking scores, such as the example of liver transplantation candidate prioritization with MELD score~\cite{wiesner2003model}. This problem is related to but different from binary decision making~\cite{narasimhan2013relationship,menon2016bipartite}. There are prior works formulating this problem in the setting of selecting the top-k items ranked based on the ranking scores for any k~\cite{celis2017ranking,yang2017measuring,zehlike2017fa,geyik2019fairness}. For each sub-problem with a specific k, the top-k ranked examples can be treated as positive while the remaining examples can be treated as negative, so that these sub-problems can be viewed as binary classification problems. There are works trying to assign a weight to each instance according to the orders and study the difference of such weights across different groups~\cite{singh2018fairness,singh2019policy}. Our focus is the fairness on bipartite ranking, which seeks for a good ranking function that ranks positive instances above negative ones~\cite{menon2016bipartite}. Kallus {\em et al}.~\cite{kallus2019fairness} defined xAUC (Area under Cross-Receiver Operating Characteristic curve) as the probability of positive examples of one group being ranked above negative examples of another group. They require equal xAUC to achieve ranking fairness. Beutel {\em et al.} proposed a similar definition of pairwise ranking fairness(PRF) as the probability that positive examples from one group are ranked above all negative examples~\cite{beutel2019fairness} and use the difference of PRF across groups as a ranking fairness metric. They further proved that some traditional fairness metrics (such as calibration and MSE) are insufficient for guaranteeing ranking fairness under PRF metric.

To address ranking fairness problem, Kallus {\em et al}.~\cite{kallus2019fairness} proposed a post-processing technique. They transformed the prediction scores in the disadvantaged group with a logistic function and optimized the empirical xAUC disparity by exhaustive searching on the space of parameters without considering the trade-off between algorithm utility and fairness. As the objective of the ranking problem is non-differentiable, there are theoretical and empirical works which propose to apply a differentiable objective to approximate the original non-differentiable objective~\cite{vogel2020learning} ~\cite{beutel2019fairness} ~\cite{narasimhan2020pairwise}.  Vogel {\em et al}. propose to use a logistic function as smooth surrogate relaxations and provide upper bounds of the difference between the global optima of the original and the relaxed objectives. Narasimhan {\em et al.} reduced ranking problems to constrained optimization problems and proposed to solve the problems by applying an existed optimization framework proposed in ~\cite{cotter2019two}. Beutel {\em et al.} proposed a pairwise regularization for the objective function~\cite{beutel2019fairness}. The regularization is computed as the absolute correlation between the residual prediction scores of the positive and negative example and the group membership of the positive example. However, PRF disparity is determined by judging whether a positive example is ranked above a negative one using an indicator function. The proposed pairwise regularization can be seen as an approximation of PRF disparity by replacing the indicator function with the residual prediction scores. This regularization does not guarantee ranking fairness under PRF metric. Moreover, it is difficult to apply this regularization to some learning methods such as boosting model~\cite{freund2003efficient}. If we apply fairness regularization proposed by Beutal \emph{et al.} to boosting model, it is challenging to reweight the samples during the boosting iterations, because the impact of increasing/decreasing the weight of the samples on fairness is difficult to control.

\begin{figure*}[htbp]
	\setlength{\abovecaptionskip}{0.2cm}
	\setlength{\belowcaptionskip}{-0.2cm}
	\center{
		\includegraphics[width=1.7\columnwidth]{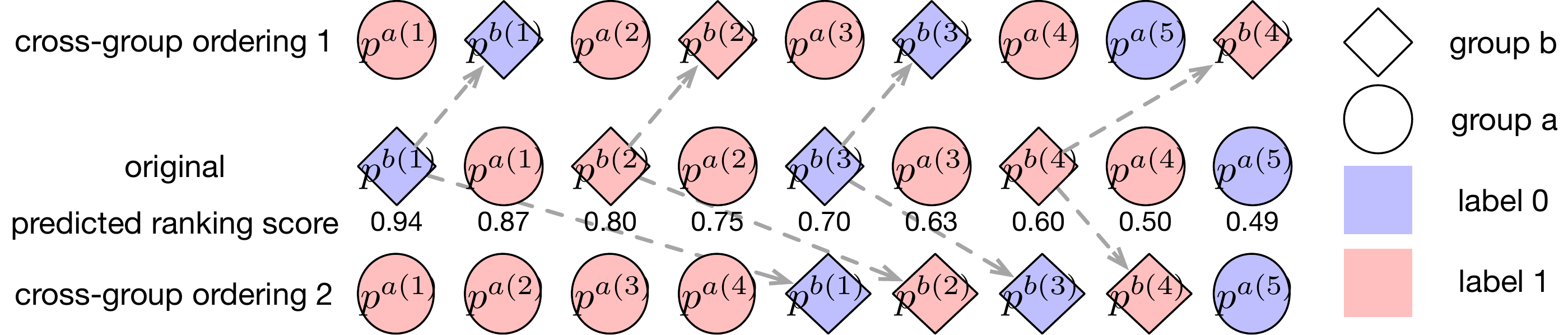}}
	\caption{An example to illustrate the post-processing. Original scores are in the middle. The first row is the ordering after post-processing, while the optimal ordering is on the bottom.}
	\label{samples}
\end{figure*}

\section{Notations and Problem Settings}
Suppose we have data ($\operatorname{X}$, $\operatorname{A}$, $\operatorname{Y}$) on features $\operatorname{X} \in \mathcal{X}$, sensitive attribute $\operatorname{A} \in \mathcal{A}$ and binary label $\operatorname{Y} \in\{0,1\}$. We are interested in the performance and the fairness issue of a predictive ranking score function $R$: $\mathcal{X} \times \mathcal{A} \rightarrow \mathbb{R}$. Here we focus on the case where $R$ returns an estimated conditional probability positive label corresponding to a given individual's ranking score. We use $ \mathrm{S}=R(\operatorname{X}, \operatorname{A})\in[0,1]$ to denote the individual ranking score variable and the score $\mathrm{S}$ means the probability that the individual belongs to positive group.

Given $\operatorname{S}$, we can derive a binary classifier with a given threshold $\theta$, such that $\hat{\operatorname{Y}}_{\theta}=\mathbb{I}[ \operatorname{S} \geq \theta]$ and $\mathbb{I}$ is the indicator function. To evaluate the performance of $R$, the receiver operator characteristic (ROC) curve is widely adopted with the false positive rate (FPR) on x-axis and the true positive rate (TPR) on y-axis as the threshold $\theta$ varies. The area under the ROC curve (AUC) quantitatively measures the quality of the learned scoring function $R$.

AUC can also be understood as the probability that a randomly drawn ranking score from the positive class is ranked above a randomly drawn score from the negative class~\cite{hanley1982meaning}

\begin{equation}
\label{eq:AUC}
\begin{aligned}
\mathrm{AUC} &= \operatorname{Pr}[\operatorname{S}_{1}>\operatorname{S}_{0}]\\
&= \frac{1}{n_{1} n_{0}} \cdot \sum\nolimits_{i:\mathrm{Y}_i=1}\sum\nolimits_{j:\mathrm{Y}_j=0}\mathbb{I}\left[R(\operatorname{X}_{i})>R(\operatorname{X}_{j})\right],
\end{aligned}
\end{equation}

where $\operatorname{S}_1$ and $\operatorname{S}_0$ represent a ranking score of a random positive and negative sample. $n_1$ and $n_0$ correspond to the number of positives and negatives, respectively. Note that we dropped the group variable in the $R$ function because it is irrelevant to the measure of AUC (i.e., $\mathrm{X}_i$, $\mathrm{X}_i$ can be from any groups). The two group-level ranking fairness metrics can be measured by the following Cross-Area Under the Curve (xAUC) metric~\cite{kallus2019fairness}.

\begin{definition}[xAUC~\cite{kallus2019fairness}] The xAUC of group $a$ over $b$ is defined as
	\begin{equation}
	\begin{aligned}
	\label{eq:xAUC}
	&\mathrm{xAUC}(a, b)=\operatorname{Pr}\left[\operatorname{S}_{1}^{a}>\operatorname{S}_{0}^{b}\right]\\
	&= \frac{1}{n^{a}_{1} n^{b}_{0}} \sum_{i:i\in a,\mathrm{Y}_{i}=1}\sum_{j: j\in b, \mathrm{Y}_{j}=0}\mathbb{I} \left[R(\operatorname{X}_{i},a) > R(\operatorname{X}_{j},b)\right],
	\end{aligned}
	\end{equation}
	where $a$ and $b$ are two groups formed by the sensitive variable $\mathrm{A}$. $\operatorname{S}_{1}^{a}$ is the ranking score of a random positive sample in $a$. $\operatorname{S}_{0}^{b}$ is the ranking score of a random negative sample in $b$. $n^a_1$ and $n^b_0$ correspond to the number of positives in a and negatives in b, respectively. $i$ is the index of a particular positive sample from group $a$, whose corresponding ranking score is $R(\operatorname{X}_{i},a)$. $j$ is the index of a particular negative sample from group $b$, whose corresponding ranking score is $R(\operatorname{X}_{j},b)$.
\end{definition}
From Eq.(\ref{eq:xAUC}) we can see that xAUC measures the probability of a random positive sample in $a$ ranked higher than a random negative sample in $b$. Correspondingly, xAUC($b$,$a$) means $\operatorname{Pr}(\operatorname{S}_{1}^{b}>\operatorname{S}_{0}^{a})$, and the ranking disparity can be measured by
\begin{equation}
\label{eq:delta_xAUC}
\begin{aligned}
\Delta \mathrm{xAUC}(a,b) &=\left|\operatorname{xAUC}(a,b)-\operatorname{xAUC}(b,a)\right|\\
&=\left|\operatorname{Pr}\left(\operatorname{S}_{1}^{a}>\operatorname{S}_{0}^{b}\right)-\operatorname{Pr}\left(\operatorname{S}_{1}^{b}>\operatorname{S}_{0}^{a}\right)\right|.
\end{aligned}
\end{equation}

\begin{definition}[Pairwise Ranking Fairness (PRF)~\cite{beutel2019fairness}] The PRF for group $a$ is defined as
	\begin{equation}
	\label{eq:pairwise_ranking}
	\begin{aligned}
	\mathrm{PRF}(a) &=\operatorname{Pr}[ \operatorname{S}_{1}^{a}>\operatorname{S}_{0}]\\
	&=\frac{1}{n^{a}_{1} \cdot n_{0}} \sum_{i:i\in a,\mathrm{Y}_{i}=1 }\sum_{j:\mathrm{Y}_{j}=0}\mathbb{I} \left[R(\operatorname{X}_{i},a) > R(\operatorname{X}_{j})\right],
	\end{aligned}
	\end{equation}
	where sample $j$ can belong to either group $a$ or group $b$.
\end{definition}


From Eq.(\ref{eq:pairwise_ranking}) we can see that the PRF for group $a$ measures the probability of a random positive sample in $a$ ranked higher than a random negative sample in either $a$ or $b$. Then we can also define the following $\Delta$PRF metric to measure the ranking disparity
\begin{equation}
\label{eq:deltaPRF}
\Delta\mathrm{PRF}(a,b)=\left|\operatorname{Pr}[ \operatorname{S}_{1}^{a}>\operatorname{S}_{0}] - \operatorname{Pr}[\operatorname{S}_{1}^{b}>\operatorname{S}_{0}]\right|.
\end{equation}

From above definitions we can see the utility (measured by AUC as in Eq.(\ref{eq:AUC})) and fairness (measured by $\Delta$xAUC in Eq.(\ref{eq:xAUC}) or $\Delta$PRF in Eq.(\ref{eq:pairwise_ranking})) of ranking function $R$ are essentially determined by the ordering of data samples induced by the predicted ranking scores. In the following, we use $\operatorname{p}^{a}$ and $\operatorname{p}^{b}$ to represent the data sample sequences in $a$ and $b$ with their ranking scores ranked in descending orders. That is, $\operatorname{p}^{a}=[\operatorname{p}^{a(1)}, \operatorname{p}^{a(2)}, ..., \operatorname{p}^{a(n^a)} ]$ with $R(\operatorname{X}_{\operatorname{p}^{a(i)}},a)\geqslant R(\operatorname{X}_{\operatorname{p}^{a(j)}},a)$ if $0\leqslant i<j\leqslant n^a$, and $\operatorname{p}^{b}$ is defined in the same way, then we have the following definition.

\begin{definition}[Cross-Group Ordering \emph{O}]
	Given ordered instance sequences $\operatorname{p}^{a}$ and $\operatorname{p}^{b}$, the \emph{cross-group ordering} $o(\operatorname{p}^{a}, \operatorname{p}^{b})$ defines a ranked list combining the instances in groups $a$ and $b$ while keeps within group instance ranking orders preserved.
\end{definition}
One example of such cross-group ordering is:

$o(\operatorname{p}^{a}, \operatorname{p}^{b})$=$[ \operatorname{p}^{a(1)}, \operatorname{p}^{b(1)}, \operatorname{p}^{a(2)}, ..., \operatorname{p}^{a(n^{a})},..., \operatorname{p}^{b(n^{b})}] $. From the above definitions we can see that we only need cross-group ordering $o(\operatorname{p}^{a}, \operatorname{p}^{b})$ to estimate both algorithm utility measured by AUC and ranking fairness measured by either $\Delta$xAUC or $\Delta$PRF, i.e., we do not need the actual ranking scores. With this definition, we have the following proposition.

\begin{proposition}
	Given ordered instance sequences $\operatorname{p}^{a}$ and $\operatorname{p}^{b}$, there exists a crossing-group ordering  $o(\operatorname{p}^{a}, \operatorname{p}^{b})$ that can achieve $\Delta \operatorname{xAUC} \leq \min(\max({1}/{n_1^b}, {1}/{n_0^b}), \max({1}/{n_1^a}, {1}/{n_0^a}))$ or $\Delta \operatorname{PRF} \leq \min(\max({n^b_0}/{(n_1^a n_0)}, {1}/{n_0}), \max({n^a_0}/{(n_1^b n_0)}, {1}/{n_0}))$ with the two ranking fairness measures.
	\label{pro1}
\end{proposition}

The proof of Proposition 1 is provided in Appendix.
From proposition ~\ref{pro1}, if we only care about ranking fairness, we can achieve a relatively low disparity by a trivial method introduced in Appendix. Our proposal in this paper is to look for an optimal cross-group ordering $o^*(\operatorname{p}^{a}, \operatorname{p}^{b})$, which can achieve ranking fairness and maximally maintain algorithm utility, through post-processing. We will use xAUC as the ranking fairness metric in our detailed derivations. The same procedure can be similarly developed for PRF based ranking fairness.

One important issue to consider is that the cross-group ordering that achieves the same level of ranking fairness is not unique. In Figure~\ref{samples}, we demonstrate an illustrative example with 9 samples showing that different cross-group ordering can result in different AUCs with the same $\Delta$xAUC. The middle row in Figure~\ref{samples} shows the original predicted ranking scores and their induced sample ranking, which achieves a ranking disparity $\Delta\mathrm{xAUC}=0.75$. The top row shows one cross-group ordering with ranking disparity $\Delta\mathrm{xAUC} = 0$ and algorithm utility $\mathrm{AUC}=0.56$. The bottom row shows another cross-group ranking with $\Delta\mathrm{xAUC} = 0$ but $\mathrm{AUC}=0.83$.

\begin{figure*}[h!]
	\setlength{\abovecaptionskip}{0.2cm}
	\setlength{\belowcaptionskip}{-0.2cm}
	\centering
	\subfigure[the ordering procedure]{
		\centering
		\includegraphics[width=0.7\columnwidth]{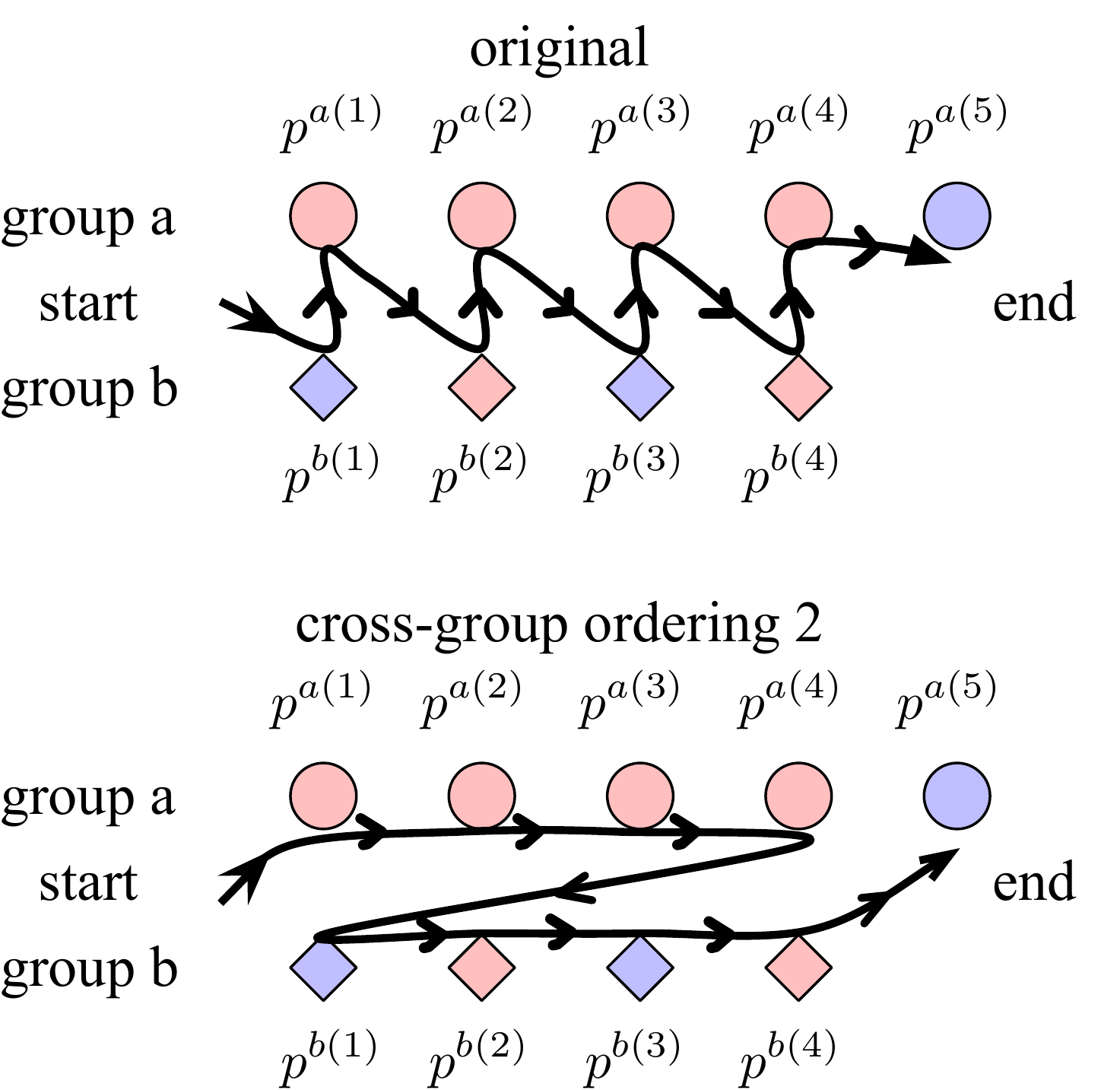}
		\label{fig:optimize_example}
	}%
	\subfigure[the optimization procedure]{
		\centering
		\includegraphics[width=1.2\columnwidth]{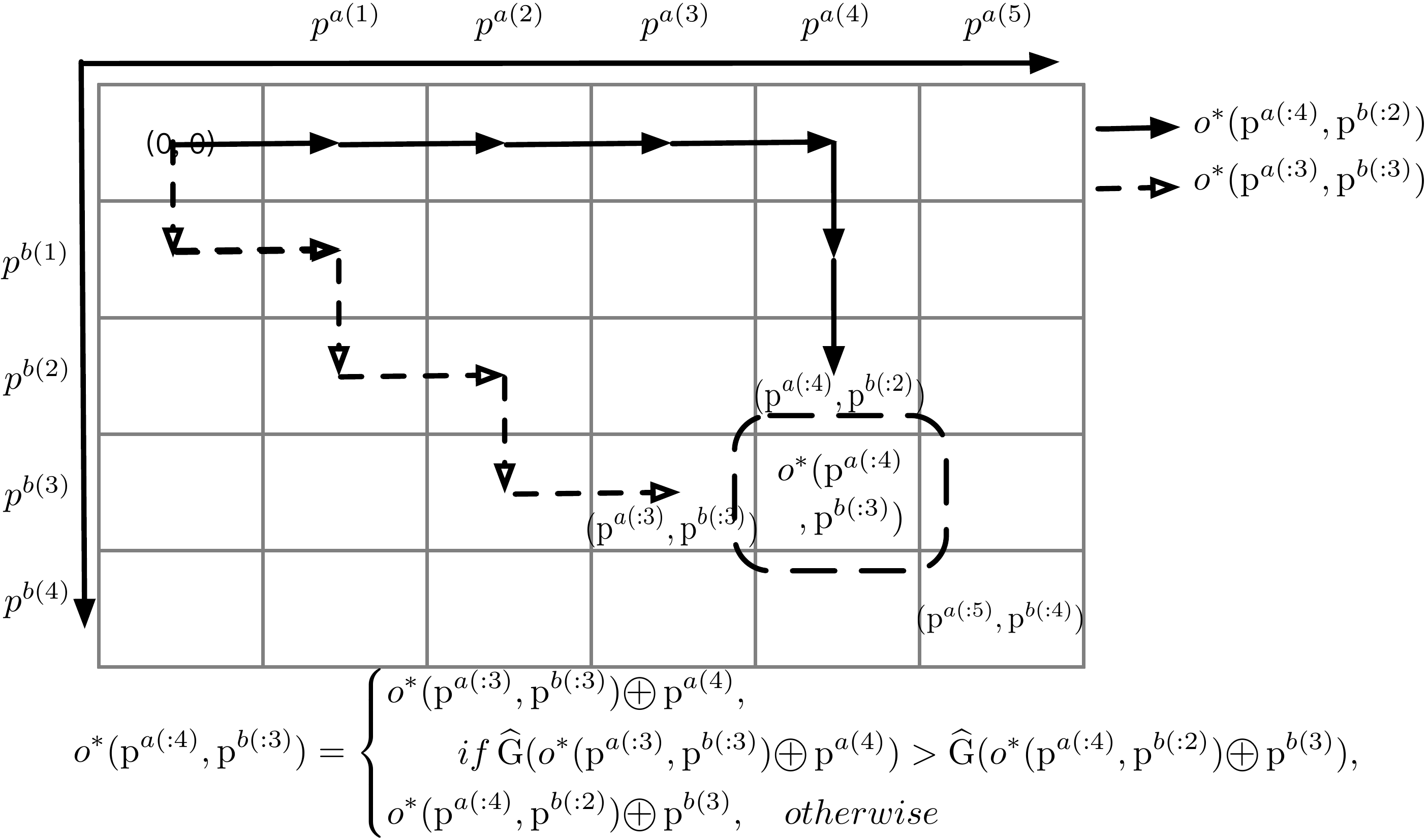}
		\label{fig:optimize_pro}
	}%
	\caption{(a).Illustration of how our proposed method optimizes such relative ordering. (b).Illustration of how our proposed method optimizes such relative ordering.}
	\label{fig:optimize_pro}
\end{figure*}

\noindent\textbf{Problem Setting}. Our goal is to identify an optimal cross-group ordering $o^*(\operatorname{p}^{a}, \operatorname{p}^{b})$ that leads to the minimum ranking disparity (measured by $\Delta$xAUC) with algorithm utility (measured by AUC) maximally maintained. We can maximize the following objective
\begin{equation}
\label{eq:optimizationobj}
J(o(\operatorname{p}^{a}, \operatorname{p}^{b}))=\text {AUC}(o(\operatorname{p}^{a}, \operatorname{p}^{b}))- \lambda \cdot\Delta\mathrm{xAUC}(o(\operatorname{p}^{a}, \operatorname{p}^{b})),
\end{equation}
where we use $\text {AUC}(o(\operatorname{p}^{a}, \operatorname{p}^{b}))$ to denote the AUC induced by the ordering $o(\operatorname{p}^{a}, \operatorname{p}^{b})$, which is calculated in the same way as in Eq.(\ref{eq:AUC}) if we think of $R(X_i)$ returning the rank of $X_i$ instead of the actual ranking score. Please note that in the rest of this paper we will use similar notations for xAUC without causing further confusions. Similarly, $\Delta\mathrm{xAUC}(o(\operatorname{p}^{a}, \operatorname{p}^{b}))$ is the ranking disparity induced by ordering $o(\operatorname{p}^{a}, \operatorname{p}^{b})$ calculated as in Eq.(\ref{eq:delta_xAUC}). Then we have the following proposition.

\begin{proposition}
	The objective function in Eq.(\ref{eq:optimizationobj}) is equivalent to:
	\begin{equation}
	\label{eq:G}
	\begin{aligned}
	&\mathrm{G}(o(\operatorname{p}^{a}, \operatorname{p}^{b}))= k_{a,b} \cdot \operatorname{xAUC}(o(\operatorname{p}^{a}, \operatorname{p}^{b}))\\
	&+ k_{b,a}  \cdot \operatorname{xAUC}(o(\operatorname{p}^{b}, \operatorname{p}^{a})) - \lambda  k \cdot\Delta\mathrm{xAUC}(o(\operatorname{p}^{a}, \operatorname{p}^{b})),
	\end{aligned}
	\end{equation}
	where $k_{a,b} = n^{a}_{1}n^{b}_{0}$, $k_{b,a} = n^{a}_{0} n^{b}_{1}$, $k = n_{0}n_{1}$, $\operatorname{xAUC}(o(\operatorname{p}^{a}, \operatorname{p}^{b}))$ indicates $ \operatorname{xAUC}(a,b)$ in Eq.~\ref{eq:xAUC} ($\operatorname{xAUC}(o(\operatorname{p}^{b}, \operatorname{p}^{a}))$ indicates $\operatorname{xAUC}(b,a)$ in Eq.~\ref{eq:xAUC}), and $\lambda$ is the hyperparameter trading off the algorithm utility and ranking disparity.
\end{proposition}
The proof of this proposition is provided in Appendix.

\section{Algorithm}

To intuitively understand the post-processing process, we treat the cross-group ordering as a path from the higher ranked instances to lower ranked instances. With the same example shown in Figure \ref{samples}, we demonstrate the original ordering of those samples induced by their predicted ranking scores (middle row of Figure \ref{samples}) on the top of Figure~\ref{fig:optimize_example}, where the direction on the path indicates the ordering. The ordering corresponding to the bottom row of Figure~\ref{samples} is demonstrated at the bottom of Figure~\ref{fig:optimize_example}.

\begin{algorithm}
	\SetAlgoLined
	\caption{xOrder: optimize the cross-group ordering with post-processing}
	\label{alg:cross-group ordering}

	\KwIn{$\lambda$, the ranking scores $\operatorname{S}^{a}$, $\operatorname{S}^{b}$ from a predictive ranking function $R$}

	Sort the ranking scores $\operatorname{S}^{a}$ and $\operatorname{S}^{b}$ in descending order. Get the instances ranking $\operatorname{p}^{a} = [\operatorname{p}^{a(1)},\operatorname{p}^{a(2)},...\operatorname{p}^{a(n^{a})}]$ and $\operatorname{p}^{b} = [\operatorname{p}^{b(1)},\operatorname{p}^{b(1)},...\operatorname{p}^{b(n^{b})}]$\\

	Initialize cross-group ordering $o^{*}(\operatorname{p}^{a(:i)}, \operatorname{p}^{b(:0)})$ $(0 \leq i \leq n^{a})$, $o^{*}(\operatorname{p}^{a(:0)}, \operatorname{p}^{b(:j)})$ $(0 \leq j \leq n^{b})$,\\
	\ForAll{i = 1, 2, 3... $n^{a}$}{
		\ForAll{j = 1, 2, 3...$n^{b}$}{
			Calculate $\operatorname{\widehat{G}}(o^{*}(\operatorname{p}^{a(:i-1)}, \operatorname{p}^{b(:j)}) \textcircled{+} \operatorname{p}^{a(:i)})$, $\operatorname{\widehat{G}}(o^{*}(\operatorname{p}^{a(:i)}, \operatorname{p}^{b(:j-1)}) \textcircled{+} \operatorname{p}^{b(:j)})$ using the Eq.(\ref{eq:G1})\\
			Update $o^{*}(\operatorname{p}^{a(:i)}, \operatorname{p}^{b(:j)})$ according to the Eq.(\ref{G2})}
	}
	\KwOut{the learnt cross-group ordering $o^{*}(\operatorname{p}^{a}, \operatorname{p}^{b})$}
\end{algorithm}
With this analogy, the optimal cross-group ordering $o^{*}(\operatorname{p}^{a}, \operatorname{p}^{b})$ can be achieved by a path finding process. The path must start from $\operatorname{p}^{a(1)}$ or $\operatorname{p}^{b(1)}$, and end with $\operatorname{p}^{a(n^a)}$ or $\operatorname{p}^{b(n^b)}$. Each instance in $\operatorname{p}^{a}$ and $\operatorname{p}^{b}$ can only appear once in the final path, and the orders of the instances in the final path must be the same as their orders in $\operatorname{p}^a$ and $\operatorname{p}^b$. The path can be obtained through a dynamic programming process. In particular, we first partition the entire decision space into a $(n^b+1)\times (n^a+1)$
grid. Each location $(i,j)~0\leqslant i\leqslant n^a,0 \leqslant j\leqslant n^b$ on the lattice corresponds to a decision step on determining whether to add $p^{a(i)}$ or $p^{b(j)}$ into the current path, which can be determined with the following rule:
\begin{equation}
\label{G2}
\begin{aligned}
& \text{Given} \quad o^{*}(\operatorname{p}^{a(:i-1)}, \operatorname{p}^{b(:j)}), o^{*}(\operatorname{p}^{a(:i)}, \operatorname{p}^{b(:j-1)})\\
& \text{if:}\ \widehat{\operatorname{G}}(o^*(\operatorname{p}^{a(:i-1)}, \operatorname{p}^{b(:j)})\textcircled{+} \operatorname{p}^{a(i)})>\widehat{\operatorname{G}}(o^*(\operatorname{p}^{a(:i)}, \operatorname{p}^{b(:j-1)})\textcircled{+} \operatorname{p}^{b(j)})\\
& \quad \quad o^{*}(\operatorname{p}^{a(:i)}, \operatorname{p}^{b(:j)}) = o^{*}(\operatorname{p}^{a(:i-1)}, \operatorname{p}^{b(:j)})\textcircled{+} \operatorname{p}^{a(i)};\\
& \text{otherwise:}\ o^{*}(\operatorname{p}^{a(:i)}, \operatorname{p}^{b(:j)}) = o^{*}(\operatorname{p}^{a(:i)}, \operatorname{p}^{b(:j-1)}) \textcircled{+} \operatorname{p}^{b(j)},
\end{aligned}
\end{equation}
where $\operatorname{p}^{a(:i)}$ represents the first $i$ elements in $\operatorname{p}^a$ ($\operatorname{p}^{a(:i-1)}$, $\operatorname{p}^{b(:j-1)}$ and $\operatorname{p}^{b(:j)}$ are similarly defined). $o^{*}(\operatorname{p}^{a(:i-1)}, \operatorname{p}^{b(:j)}) \textcircled{+} \operatorname{p}^{a(i)}$ means appending $\operatorname{p}^{a(i)}$ to the end of $o^*(\operatorname{p}^{a(:i-1)}, \operatorname{p}^{b(:j)})$. The value of function $\widehat{\mathrm{G}}(o(\operatorname{p}^{a(:i)},\operatorname{p}^{b(:j)}))$ is defined as follows

\begin{equation}
\label{eq:G1}
\begin{aligned}
&\widehat{\operatorname{G}}(o(\operatorname{p}^{a(:i)}, \operatorname{p}^{b(:j)})) =\\
&k_{a,b} \cdot {\operatorname{xAUC}}(o(\operatorname{p}^{a(:i)}, \operatorname{p}^{b(:j)}) \textcircled{+} \operatorname{p}^{b(j+1:n^b)})\\
&+k_{b,a} \cdot {\operatorname{xAUC}}(o(\operatorname{p}^{b(:j)}, \operatorname{p}^{a(:i)}) \textcircled{+} \operatorname{p}^{a(i+1:n^a)})\\
&+\lambda k\cdot |{\operatorname{xAUC}}(o(\operatorname{p}^{a(:i)}, \operatorname{p}^{b(:j)}) \textcircled{+} \operatorname{p}^{b(j+1:n^b)})\\
&-{\operatorname{xAUC}}(o(\operatorname{p}^{b(:j)}, \operatorname{p}^{a(:i)}) \textcircled{+} \operatorname{p}^{a(i+1:n^a)})|,
\end{aligned}
\end{equation}

in which $\widehat{\operatorname{G}}(o(\operatorname{p}^{a(:i)}, \operatorname{p}^{b(:j)})) = \operatorname{G}(o(\operatorname{p}^{a}, \operatorname{p}^{b}))$ in Eq.~\ref{eq:G} when $i = n^a,j=n^b$.

Figure \ref{fig:optimize_example} demonstrates the case of applying our rule to step $i=4,j=3$ in the example shown in Figure \ref{samples}.

Algorithm \ref{alg:cross-group ordering} summarized the whole pipeline of identifying the optimal path. In particular, our algorithm calculates a cost for every point $(i,j)$ in the decision lattice as the value of the $\operatorname{\widehat{G}}$ function evaluated on the path reaching $(i,j)$ from $(0,0)$. Please note that the first row ($j=0$) only involves the instances in $a$, therefore the path reaching the points in this row are uniquely defined considering the within group instance order should be preserved in the path. The decision points in the first column ($i=0$) enjoy similar characteristics. After the cost values for the decision points in the first row and column are calculated, the costs values on the rest of the decision points in the lattice can be calculated iteratively until $i=n^a$ and $j=n^b$.

Algorithm 1 can also be viewed as a process to maximize the objective function in Eq.(\ref{eq:optimizationobj}). It has $O(N^2)$ time complexity as it is a 2-D dynamic programming process. Different $\lambda$ values trade-off the algorithm utility and ranking fairness differently and the solution Algorithm 1 converges to is a local optima. Moreover, we have the following proposition.

\begin{theorem}
	xOrder can achieve the global optimal solution of maximizing Eq.(\ref{eq:optimizationobj}) with $\lambda=0$. xOrder has the upper bounds of fairness disparities defined in Eq.(\ref{eq:delta_xAUC}) and Eq.(\ref{eq:deltaPRF}) as $\lambda$ approaches infinity:

	\begin{equation}
	\begin{aligned}
	&\Delta \mathrm{x} \mathrm{AUC} \leq \max \left(\frac{1}{n_{1}^{a}}, \frac{1}{n_{1}^{b}}\right)\\
	&\Delta \mathrm{PRF} \leq \max \left(\frac{n_{0}^{b}}{n_{0} \cdot n_{1}^{a}}, \frac{n_{0}^{a}}{n_{0} \cdot n_{1}^{b}}\right)
	\end{aligned}
	\end{equation}

\end{theorem}
The proof of this theorem is provided in Appendix.

\begin{table} \scriptsize
    \setlength{\abovecaptionskip}{0cm}
    \setlength{\belowcaptionskip}{-.2cm}
	\caption{Summary of benchmark data sets}.
	\label{table:summary_eicu}
	\centering
	\begin{tabular}{c|c|c|c|c}
		\hline
		Dataset & $n$ & $p$ & $\mathrm{A}$ & $\mathrm{Y}$ \\
		\hline
		COMPAS\cite{angwin2016machine} & 6,167 & 400 & Race(white, non-white) & Non-recidivism within 2 years \\
		Adult\cite{kohavi1996scaling} & 30,162 & 98 & Race(white, non-white) & Income $\geq$ 50K\\
		Framingham~\cite{levy199950} & 4,658 & 7 & Gender(male,female) & 10-year CHD incidence\\
		MIMIC\cite{johnson2016mimic} & 21,139 & 714 & Gender(male,female) & Mortality\\
		MIMIC\cite{johnson2016mimic} & 21,139 & 714 & Race(white, non-white) & Prolonged length of stay\\
		eICU\cite{pollard2018eicu} & 17,402 & 60 & Race(white, non-white) & Prolonged length of stay\\
		\hline
	\end{tabular}
	\vspace{-1.em}
\end{table}

\begin{figure*}[htbp]
	\setlength{\abovecaptionskip}{0pt}
	\setlength{\belowcaptionskip}{-0.0cm}
	\centering
	\includegraphics[width=1.6\columnwidth]{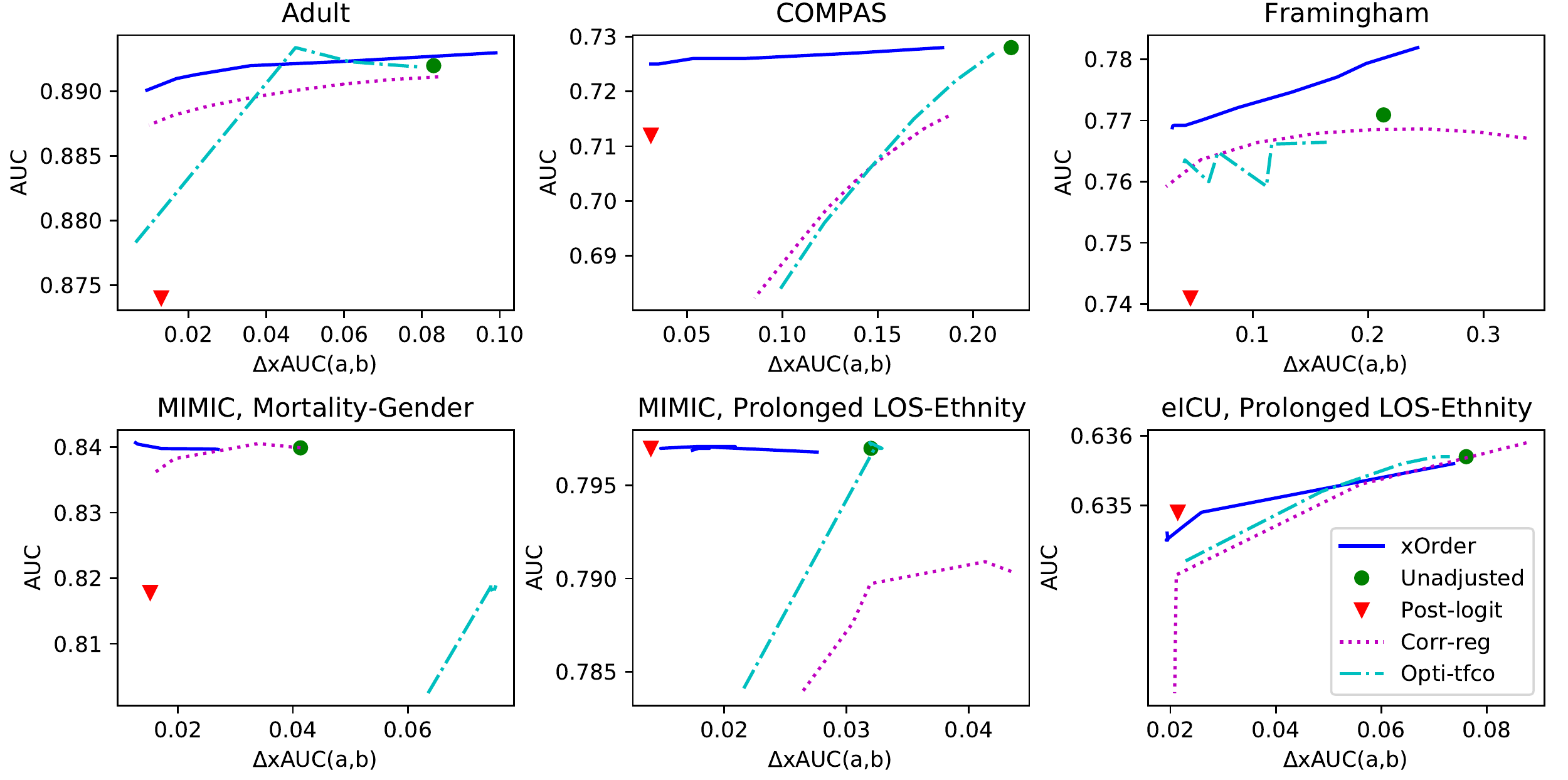}
	\caption{$\mathrm{AUC}$-$\mathrm{\Delta xAUC}$ trade-off with linear model.}
	\label{fig:lr_result}
\end{figure*}
\noindent\textbf{The testing phase}. Since the learned ordering in training phase cannot be directly used in testing stage, we propose to transfer the information of the learned ordering by rearranging the ranking scores of the disadvantaged group (which is assumed to be group $b$ without the loss of generality). If we assume the distribution of predicted ranking scores is the same on the training and test sets, the same scores will have the same quantiles in both sets. It means that by rearranging the ranking scores through interpolation, we can transfer the ordering learned by \texttt{xOrder} on training set. In particular, the process contains two steps:
\begin{enumerate}[leftmargin=*]
	\item \underline{\emph{Rank score adjustment for the group b in training data}}. Fixing the ranking scores for training instances in group $a$, the adjusted ranking scores for training instances in group $b$ will be obtained by uniform linear interpolation according to their relative positions in the learned ordering. For example, if we have an ordered sequence $(\operatorname{p}^{a(1)},\operatorname{p}^{b(1)},\operatorname{p}^{b(2)},\operatorname{p}^{a(2)})$ with the ranking scores for $\operatorname{p}^{a(1)}$ and $\operatorname{p}^{a(2)}$ being 0.8 and 0.5, then the adjusted ranking scores for $\operatorname{p}^{b(1)}$ and $\operatorname{p}^{b(2)}$ being 0.7 and 0.6.

	\item \underline{\emph{Rank score adjustment for the group b in test data}}. For testing instances, we follow the same practice of just adjusting the ranking scores of the instances from group $b$ but keep the ranking scores for instances from group $a$ unchanged. \sen{We propose a proportional interpolation for the adjustment process which has O(N) time complexity. In particular, we first rank training instances from group $b$ according to their raw unadjusted ranking scores to get an ordered list. Then the adjusted ranking scores for testing instances in $b$ can be obtained by a linear transformation. For example, if we want to adjust the testing original score $\operatorname{p_{te}}^{b(i)}$ given training original ordered sequence $(\operatorname{p}^{b(1)},\operatorname{p}^{b(2)}$ being 0.8 and 0.5 and the training adjusted ordered sequence $(\operatorname{\hat{p}}^{b(1)},\operatorname{\hat{p}}^{b(2)}$ being 0.7 and 0.4, then the quantile of the testing adjusted ranking score $\operatorname{\hat{p}_{te}}^{b(i)}$ in [0.4, 0.7] equals to the quantile of $\operatorname{p_{te}}^{b(i)}$ in [0.5, 0.8]:}
	$\frac{0.7 - \operatorname{\hat{p}_{te}}^{b(i)}}{0.7-0.4} = \frac{0.8 - \operatorname{p_{te}}^{b(i)}}{0.8-0.5}.$
\end{enumerate}

\sen{From the above process, we transfer the learned ordering from training set to test set by interpolation. We also try other operations to achieve interpolation (e.g. ranking uniform linear interpolation) and get similar experimental results.}

\begin{figure*}[htbp]
	\setlength{\abovecaptionskip}{0pt}
	\setlength{\belowcaptionskip}{-.1cm}
	\centering
	\includegraphics[width=1.9\columnwidth]{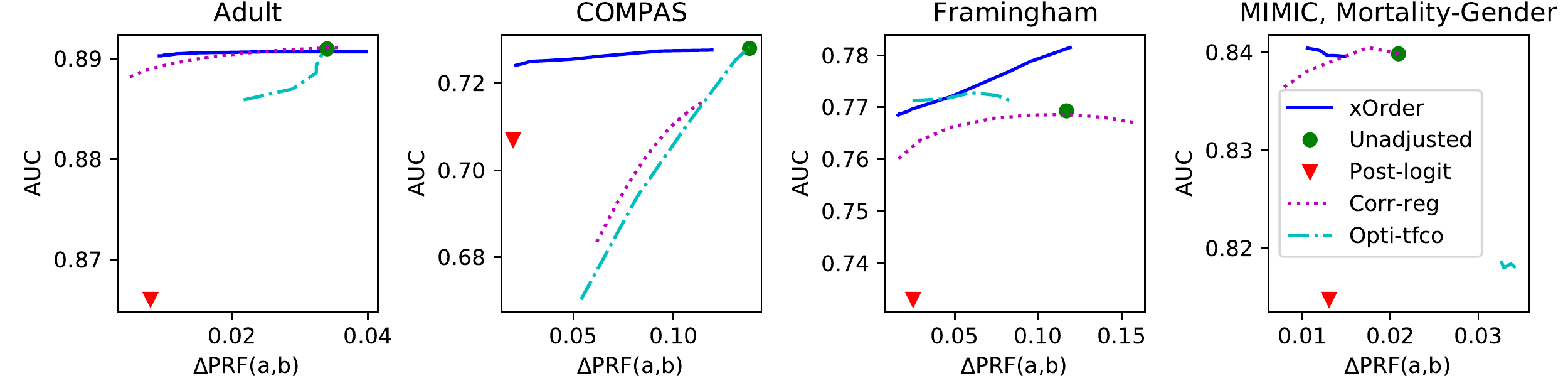}
	\caption{$\mathrm{AUC}$-$\Delta \mathrm{PRF}$ trade-off with linear model.}
	\label{fig:pr_result}
	\vspace{-.1em}
\end{figure*}

\begin{figure*}[htbp]
	\setlength{\abovecaptionskip}{0pt}
	\setlength{\belowcaptionskip}{-0.2cm}
	\centering
	\includegraphics[width=1.9\columnwidth]{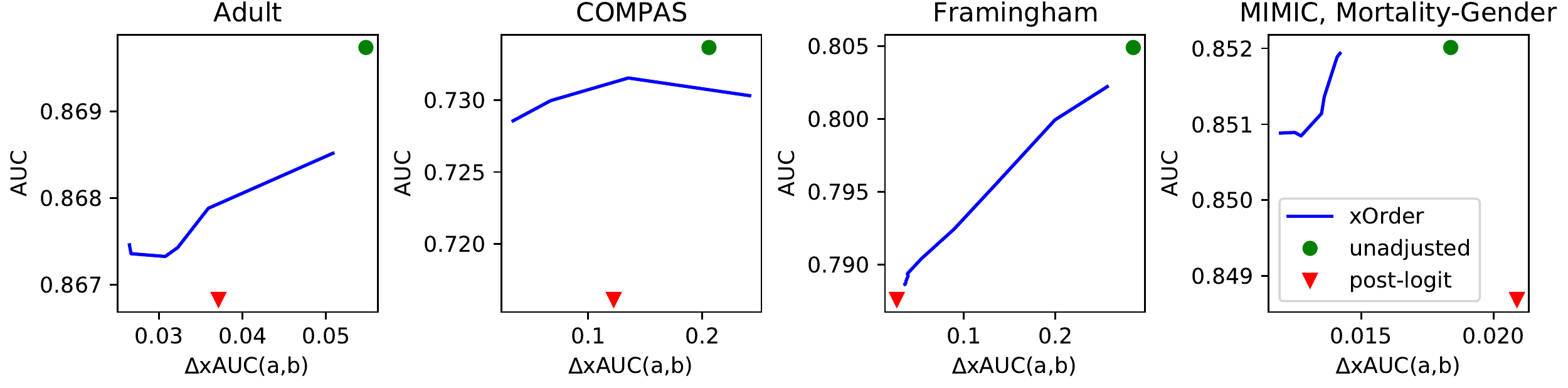}
	\caption{$\mathrm{AUC}$-$\Delta \mathrm{xAUC}$ trade-off with bipartite rankboost model.}
	\label{fig:rb_result}
	\vspace{-.2em}
\end{figure*}

\section{Experiment}
\subsection{Data Sets and Baselines}
\noindent\textbf{Data sets}. We conduct experiments on 4 popular benchmark data sets and two real-world clinical data sets for studying algorithm fairness. For each data set, we randomly select 70\% of the data set as train set and the remaining as test set following the setting in~\cite{kallus2019fairness}. We use the non-sensitive features as the input. The basic information for these data sets is summarized in Table \ref{table:summary_eicu}, where $n$ and $p$ are the number of instances and features, and CHD is the abbreviation of coronary heart disease. To study ranking fairness in real world clinical ranking prediction scenarios, we analyze the unfair phenomena and conduct experiments on two clinical data sets MIMIC-III and eICU. As in Table \ref{table:summary_eicu}, MIMIC-III a real world electronic health record repository for ICU patients~\cite{johnson2016mimic} The data set was preprocessed as in~\cite{harutyunyan2019multitask} with $n$ = 21,139 and $p$ = 714 ($p$ is the number of features). Each instance is a specific ICU stay. eICU is another real world dataset of ICU electronic health records~\cite{pollard2018eicu}. It is was preprocessed with $n$ = 17,402 and $p$ = 60. We consider the same setting of the label and protected variable for MIMIC and eICU. For label $\operatorname{Y}$, we consider in-hospital mortality and prolonged length of stay (whether the ICU stay is longer than 1 week). For protected variable $\operatorname{A}$, we consider gender (male, female) and ethnicity (white, non-white).

\noindent\textbf{Baselines}. In this paper, $\Delta \mathrm{xAUC}$ and $\Delta \mathrm{PRF}$ are used as ranking fairness metrics and $\mathrm{AUC}$ is used to measure the algorithm utility. To verify the validity of our algorithm, we adopt 2 base models: linear model \cite{narasimhan2020pairwise} \cite{kallus2019fairness} and rankboost~\cite{freund2003efficient} in our empirical evaluations. The algorithms that originally proposed xAUC~\cite{kallus2019fairness} PRF~\cite{beutel2019fairness} as well as the framework proposed in ~\cite{narasimhan2020pairwise} are evaluated as baselines, which are denoted as \textbf{post-logit}, \textbf{corr-reg} and \textbf{opti-tfco}. We also report the performance obtained by the two original base models without fairness considerations (called \textbf{unadjusted}). post-logit and \texttt{xOrder} are post-processing algorithms and can be used to all base models.

We propose the following training procedure for linear model. For corr-reg, the model is trained by optimizing the weighted sum of the cross-entropy loss and fairness regularization originally proposed in ~\cite{beutel2019fairness} with gradient descent. Since the constrained optimization problem in opti-tfco can not be solved with gradient descent, we optimize it with a specific optimization framework called tfco \cite{cotter2019two} as the original implementation did \cite{narasimhan2020pairwise}. For post-processing algorithms including post-logit and \texttt{xOrder}, we train the linear model on training data without any specific considerations on ranking fairness to obtain the unadjusted prediction ranking scores. The model is optimized with two optimization methods: 1)optimizing cross-entropy with gradient descent; 2) solving an unconstrained problem to maximize utility with tfco. For the rankboost, as it is optimized by boosting iterations, it is not compatible with corr-reg and opti-tfco. We only report the results of post-logit and \texttt{xOrder}.

The total $\mathrm{AUC}$ and ranking fairness metrics on test data are reported. In addition, we plot a curve showing the trade-off between ranking fairness and algorithm utility for \texttt{xOrder}, corr-reg and opti-tfco with varying the trade-off parameter. \pan{For post-logit, we follow the procedure in its original paper \cite{kallus2019fairness} to choose the optimal parameters with lowest disparity and report the corresponding result.} We repeat each setting of the experiment ten times and report the average results. To make the comparisons fair, we report the results of different methods under the same base model separately.

\subsection{Evaluation Results}
For the three benchmark data sets (Adult, COMPAS, Framingham), we conduct the experiments with $\operatorname{Y}-\operatorname{A}$ combinations shown in Table \ref{table:summary_eicu}. For MIMIC-III, we show the results as $\operatorname{Y}-\operatorname{A}$ combinations (mortality-gender and prolonged length of stay (LOS)-ethnicity). For eICU, we show the results as $\operatorname{Y}-\operatorname{A}$ combination prolonged length of stay (LOS)-ethnicity. In the experiments, we find that there is no significant unfairness when we focus on prolonged length of stay prediction with sensitive attribute gender. However, we observe an obvious disparity on prolonged length of stay prediction with sensitive attribute ethnicity. We guess the phenomena are in line with common sense, since the difference in economic conditions is more reflected in race rather than gender for patients living in the ICU.

The results are shown in Figure \ref{fig:lr_result}-\ref{fig:rb_result}. We show the result on all data sets using linear model with the metric $\Delta \mathrm{xAUC}$ while we show the experiments with the metric $\Delta \mathrm{PRF}$ on the 4 data sets. For the results of post-processing algorithms with linear model under two optimization methods, we report them with a higher average AUC of the unadjusted results in our main text. It also indicates that post-processing algorithms are model agnostic which have the advantage to choose a better base model.

\begin{figure*}[h!]
	\setlength{\abovecaptionskip}{0pt}
	\setlength{\belowcaptionskip}{-0.2cm}
	\centering
	\subfigure[]{
		\centering
		\includegraphics[width=0.95\columnwidth]{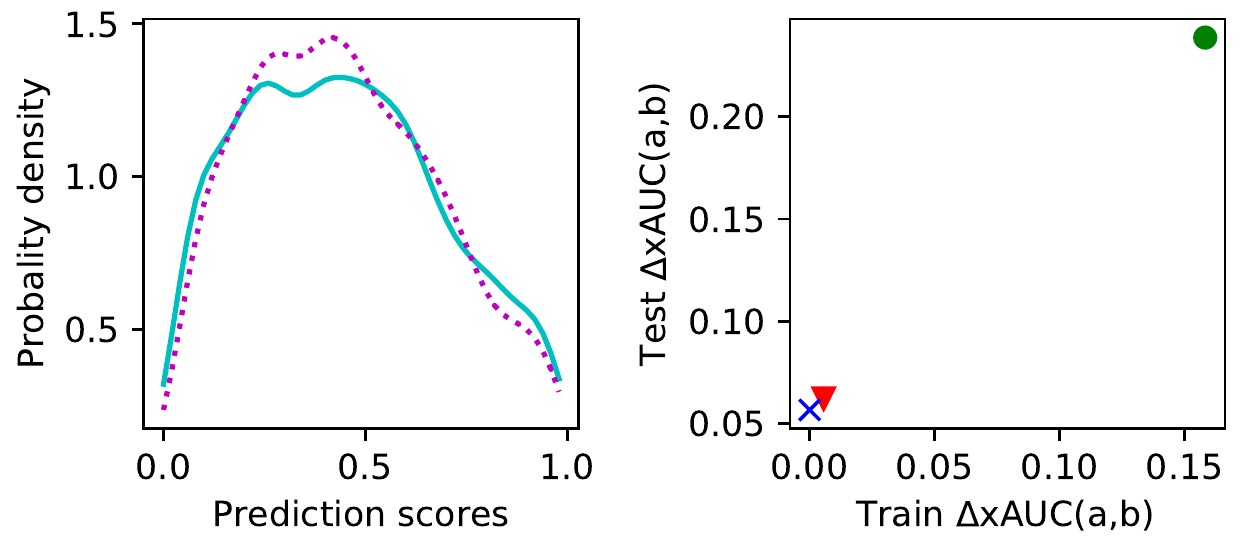}
	}%
	\subfigure[]{
		\centering
		\includegraphics[width=0.95\columnwidth]{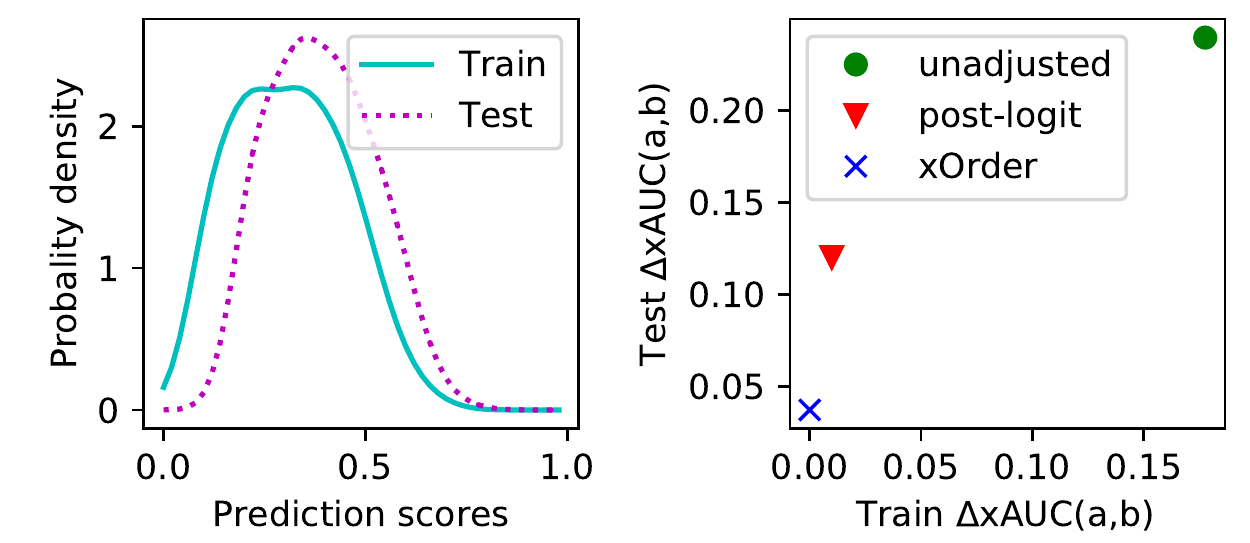}
	}%
	\caption{Result analysis on COMPAS data set with $\Delta \mathrm{xAUC}$ metric. (a) illustrates the result with linear model. (b) illustrates the result with bipartite rankboost model. The left part of each sub-figure is the distributions of prediction scores on training and test data, the right part of each sub-figure plots $\Delta \mathrm{xAUC}$ on training data test data.}
	\label{fig:ana_example}
\end{figure*}

\begin{figure*}[h]
	\setlength{\abovecaptionskip}{0pt}
	\setlength{\belowcaptionskip}{-0.2cm}
	\centering
	\subfigure[]{
		\centering
		\includegraphics[width=0.95\columnwidth]{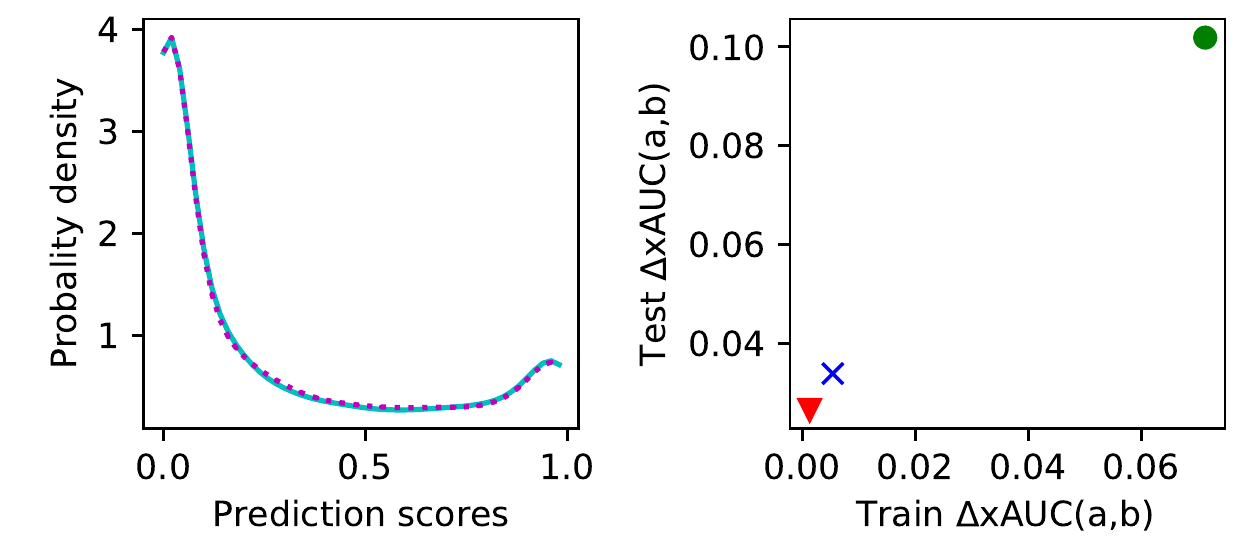}
	}%
	\subfigure[]{
		\centering
		\includegraphics[width=0.95\columnwidth]{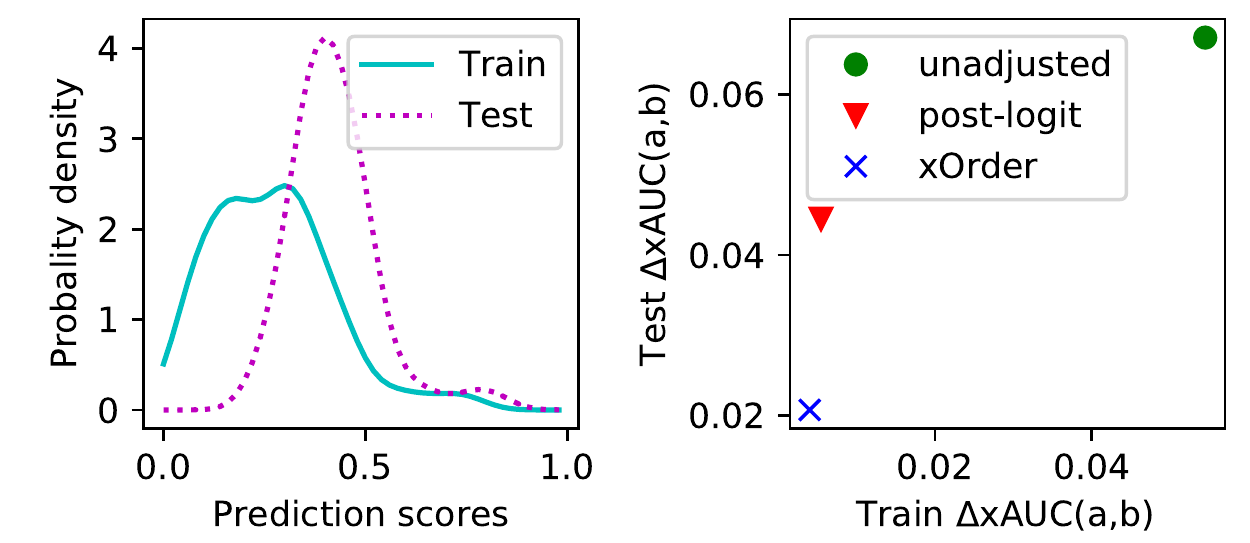}
	}%
	\caption{Result analysis on Adult data set with $\Delta \mathrm{xAUC}$ metric. (a) illustrates the result with linear model. (b) illustrates the result with bipartite rankboost model. The left part of each sub-figure is the distributions of prediction scores on training and test data, the right part of each sub-figure plots $\Delta \mathrm{xAUC}$ on training data and test data.}
	\label{fig:ana_example_adult}
\end{figure*}

\subsubsection{Linear Model with Metric $\Delta \mathrm{xAUC}$}
All four methods considering ranking fairness (\texttt{xOrder}, corr-reg, post-logit and opti-tfco) are able to obtain lower $\Delta \mathrm{xAUC}$ comparing to the unadjusted results. \texttt{xOrder} and post-logit can achieve $\Delta \mathrm{xAUC}$ close to zero on almost all the datasets. This supports proposition 1 empirically that post-processing by changing cross-group ordering has the potential to achieve $\Delta \mathrm{xAUC}$ closed to zero. corr-reg and opti-tfco fail to obtain results with low ranking disparities on COMPAS and MIMIC-III. One possible reason is that the correlation regularizer is only an approximation of ranking disparity.  
Another observation is that \texttt{xOrder} can achieve a better trade-off between utility and fairness. It can obtain competitive or obviously better $\mathrm{AUC}$ under the same level of $\Delta \mathrm{xAUC}$ when compared with other methods, especially in the regiment with low $\Delta \mathrm{xAUC}$. corr-reg performs worse than other methods in COMPAS and MIMIC-III (Prolonged LOS-Ethnicity), while post-logit performs worse in Adult and Fragmingham and opti-tfco performs worse in COMPAS and MIMIC-III (Mortality-Gender). \texttt{xOrder} performs well consistently on all data sets.

\subsubsection{Linear Model with Metric $\Delta \mathrm{PRF}$}
Figure \ref{fig:pr_result} illustrates the results with linear model and $\Delta \mathrm{PRF}$ metric on four data sets. The findings are similar to those in Figure \ref{fig:lr_result}, since $\Delta \mathrm{PRF}$ and $\Delta \mathrm{xAUC}$ are highly correlated on these data sets. Although the regularization term in corr-reg is related to the definition of $\Delta \mathrm{PRF}$, \texttt{xOrder} maintains its superiority over corr-reg. For opti-tfco, it fails to reduce the disparity on Adult and MIMIC-III(Mortality-Gender).

\subsubsection{Rankboost with Metric $\Delta \mathrm{xAUC}$}
The results on bipartite rankboost model and the associated $\Delta \mathrm{xAUC}$ are shown in Figure \ref{fig:rb_result}. It can still be observed that \texttt{xOrder} achieves higher $\mathrm{AUC}$ than post-logit under the same $\Delta \mathrm{xAUC}$. Moreover, post-logit cannot achieve $\Delta \mathrm{xAUC}$ as low as \texttt{xOrder} on adult and COMPAS. This is because the distributions of the prediction scores from the bipartite rankboost model on training and test data are significantly different. Therefore, the function in post-logit which achieves equal $\mathrm{xAUC}$ on training data may not generalize well on test data. Our method is robust against such differences. We empirically illustrate the relation between the distribution of the prediction ranking scores and the generalization ability of post-progressing algorithms in the next section.

\begin{figure}[htbp]
	\setlength{\abovecaptionskip}{0.2cm}
	\setlength{\belowcaptionskip}{-0.5cm}
	\centering
	\includegraphics[width=0.9\columnwidth]{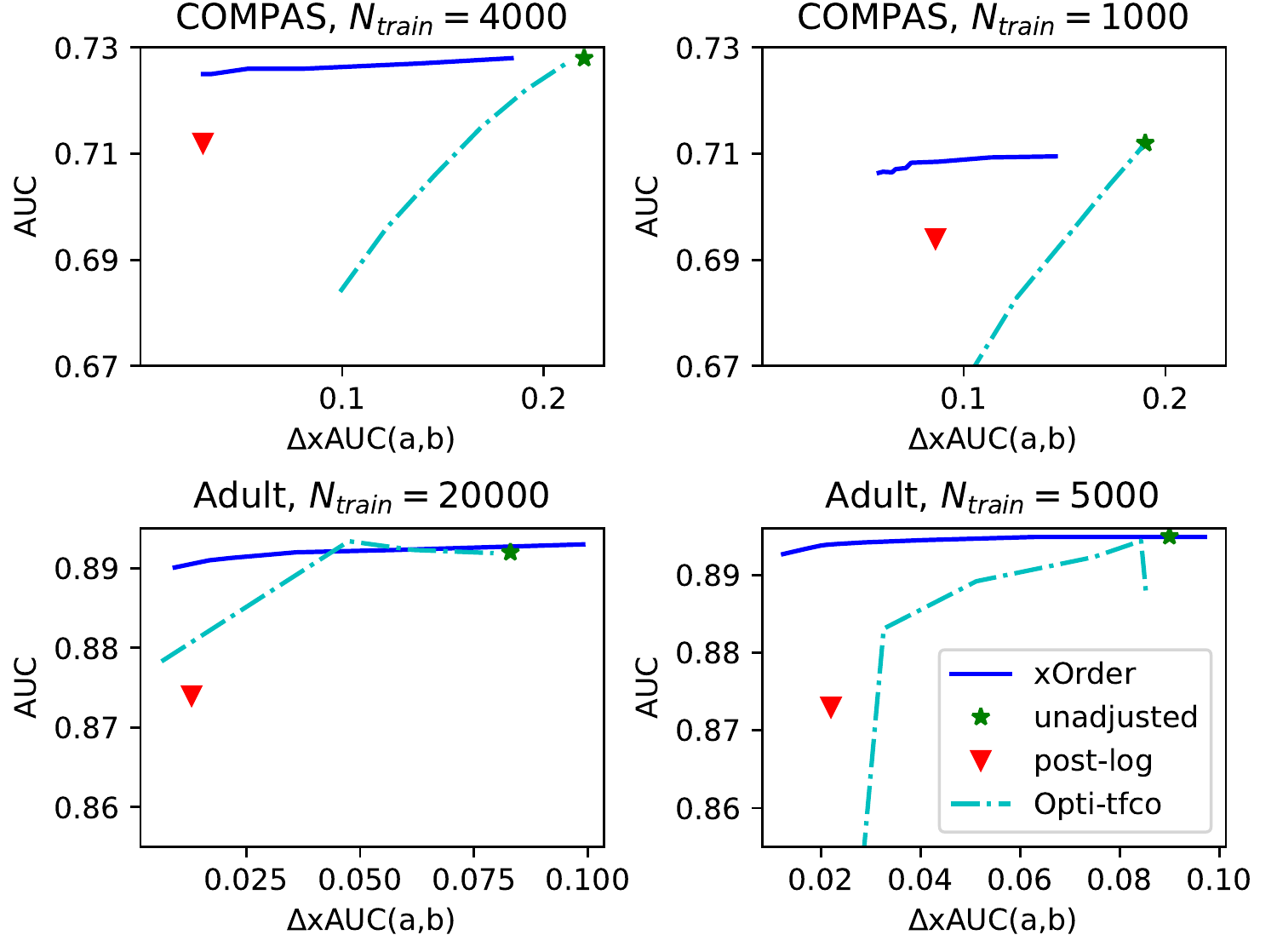}
	\caption{$\mathrm{AUC}$-$\mathrm{\Delta xAUC}$ trade-offs using linear model optimizing AUC with fewer training sample.}
	\label{fig:tr_num_ana}
	\vspace{-.5em}
\end{figure}

\section{Robustness Discussion}

\noindent \textbf{fewer training samples.} To assess the performance of our algorithm given fewer training samples, we conduct the experiments with the different number of training samples, in which we select linear model as the base model with the optimization framework in \cite{narasimhan2020pairwise}. From Figure \ref{fig:tr_num_ana}, as the number of training samples changed from 4000 to 1000, \texttt{xOrder} achieves a lower disparity compared to baselines. Meanwhile, \texttt{xOrder} still achieves a lower disparity while maintains a maximum algorithm utility. Experiment results on Adult data set shown in Figure \ref{fig:tr_num_ana} also confirm the statement. As the number of training samples reduced to 5000, \texttt{xOrder} realizes the lowest disparity($\Delta \mathrm{xAUC} = 0.01$) compared to baselines($\Delta \mathrm{xAUC} = 0.02$). Similarly, \texttt{xOrder} keeps the minimum disparity while maintains the maximum AUC on the 4 experiments. From the two series experiments on COMPAS and Adult, \texttt{xOrder} has stable performances and maintains its advantages that realize a maximum algorithm utility and a minimum ranking disparity.

\noindent \textbf{different score distributions.} To compare the robustness of the two post-processing algorithms \texttt{xOrder} and post-logit when faced with the difference between training and test ranking score distributions, we implement experiments on two data sets with 2 base models (linear model and RankBoost) on $\Delta \mathrm{xAUC}$ metric. According to Figure \ref{fig:rb_result}, post-logit fails to achieve $\Delta \mathrm{xAUC}$ as low as \texttt{xOrder} with bipartite rankboost model, while both methods can achieve low $\Delta \mathrm{xAUC}$ with linear model as shown in Figure \ref{fig:lr_result}. To analyze this phenomenon, we use COMPAS and adult as examples in Figure \ref{fig:ana_example} and \ref{fig:ana_example_adult}. For different models, we illustrate the distributions of prediction scores $\operatorname{S}$ on training and test data. We further plot $\Delta \mathrm{xAUC}$ on training data versus $ \Delta \mathrm{xAUC}$ on test data. With linear model, the distributions of $\operatorname{S}$ on training and test data are closed to each other. In this situation, the transform relations learnt from post-logit and \texttt{xOrder} can both obtain results with low $\Delta \mathrm{xAUC}$ on test data. While the distributions of the scores $\operatorname{S}$ on training and test data become different, the function learned from post-logit may not be generalized well on test data to achieve low $\Delta \mathrm{xAUC}$. Similar results can be observed on adult data set in Figure \ref{fig:ana_example_adult}. These phenomena occur in repeat experiments on both data sets. We guess the reason is that \texttt{xOrder} adjusts the relative ordering and it is more flexible than post-logit which optimizes a logistic class function. According to the experiments, \texttt{xOrder} is more robust to such distribution difference.

\section{Conclusion}
In this paper, we investigate the problem of algorithmic fairness. Given an imperfect model, we propose a general post-processing framework, \texttt{xOrder}, to achieve a good balance between ranking fairness and model utility by direct adjustment to the cross group ranking orders. We formulate \texttt{xOrder} as an optimization problem and propose a dynamic programming process to solve it. Empirical results on both benchmark and real world medical data sets demonstrated that \texttt{xOrder} can achieve a low ranking disparity while keeping a maximum algorithm utility.

\begin{acks}
Sen Cui, Weishen Pan and Changshui Zhang would like to acknowledge the funding by the National Key Research and Development Program of China (No. 2018AAA0100701) and Beijing Academy of Artificial Intelligence (BAAI). Fei Wang would like to acknowledge the support from Amazon Web Service (AWS) Machine Learning for Research Award and Google Faculty Research Award.
\end{acks}


\bibliographystyle{ACM-Reference-Format}
\bibliography{main}


\begin{thebibliography}{42}


\ifx \showCODEN    \undefined \def \showCODEN     #1{\unskip}     \fi
\ifx \showDOI      \undefined \def \showDOI       #1{#1}\fi
\ifx \showISBNx    \undefined \def \showISBNx     #1{\unskip}     \fi
\ifx \showISBNxiii \undefined \def \showISBNxiii  #1{\unskip}     \fi
\ifx \showISSN     \undefined \def \showISSN      #1{\unskip}     \fi
\ifx \showLCCN     \undefined \def \showLCCN      #1{\unskip}     \fi
\ifx \shownote     \undefined \def \shownote      #1{#1}          \fi
\ifx \showarticletitle \undefined \def \showarticletitle #1{#1}   \fi
\ifx \showURL      \undefined \def \showURL       {\relax}        \fi
\providecommand\bibfield[2]{#2}
\providecommand\bibinfo[2]{#2}
\providecommand\natexlab[1]{#1}
\providecommand\showeprint[2][]{arXiv:#2}

\bibitem[\protect\citeauthoryear{Angwin, Larson, Mattu, and Kirchner}{Angwin
  et~al\mbox{.}}{2016}]%
        {angwin2016machine}
\bibfield{author}{\bibinfo{person}{Julia Angwin}, \bibinfo{person}{Jeff
  Larson}, \bibinfo{person}{Surya Mattu}, {and} \bibinfo{person}{Lauren
  Kirchner}.} \bibinfo{year}{2016}\natexlab{}.
\newblock \showarticletitle{Machine bias: There’s software used across the
  country to predict future criminals}.
\newblock \bibinfo{journal}{\emph{And it’s biased against blacks.
  ProPublica}}  \bibinfo{volume}{23} (\bibinfo{year}{2016}).
\newblock


\bibitem[\protect\citeauthoryear{Beutel, Chen, Doshi, Qian, Wei, Wu, Heldt,
  Zhao, Hong, Chi, et~al\mbox{.}}{Beutel et~al\mbox{.}}{2019a}]%
        {beutel2019fairness}
\bibfield{author}{\bibinfo{person}{Alex Beutel}, \bibinfo{person}{Jilin Chen},
  \bibinfo{person}{Tulsee Doshi}, \bibinfo{person}{Hai Qian},
  \bibinfo{person}{Li Wei}, \bibinfo{person}{Yi Wu}, \bibinfo{person}{Lukasz
  Heldt}, \bibinfo{person}{Zhe Zhao}, \bibinfo{person}{Lichan Hong},
  \bibinfo{person}{Ed~H Chi}, {et~al\mbox{.}}}
  \bibinfo{year}{2019}\natexlab{a}.
\newblock \showarticletitle{Fairness in recommendation ranking through pairwise
  comparisons}. In \bibinfo{booktitle}{\emph{Proceedings of the 25th ACM SIGKDD
  International Conference on Knowledge Discovery \& Data Mining}}.
  \bibinfo{pages}{2212--2220}.
\newblock


\bibitem[\protect\citeauthoryear{Beutel, Chen, Doshi, Qian, Woodruff, Luu,
  Kreitmann, Bischof, and Chi}{Beutel et~al\mbox{.}}{2019b}]%
        {beutel2019putting}
\bibfield{author}{\bibinfo{person}{Alex Beutel}, \bibinfo{person}{Jilin Chen},
  \bibinfo{person}{Tulsee Doshi}, \bibinfo{person}{Hai Qian},
  \bibinfo{person}{Allison Woodruff}, \bibinfo{person}{Christine Luu},
  \bibinfo{person}{Pierre Kreitmann}, \bibinfo{person}{Jonathan Bischof}, {and}
  \bibinfo{person}{Ed~H Chi}.} \bibinfo{year}{2019}\natexlab{b}.
\newblock \showarticletitle{Putting fairness principles into practice:
  Challenges, metrics, and improvements}. In
  \bibinfo{booktitle}{\emph{Proceedings of the 2019 AAAI/ACM Conference on AI,
  Ethics, and Society}}. \bibinfo{pages}{453--459}.
\newblock


\bibitem[\protect\citeauthoryear{Beutel, Chen, Zhao, and Chi}{Beutel
  et~al\mbox{.}}{2017}]%
        {beutel2017data}
\bibfield{author}{\bibinfo{person}{Alex Beutel}, \bibinfo{person}{Jilin Chen},
  \bibinfo{person}{Zhe Zhao}, {and} \bibinfo{person}{Ed~H Chi}.}
  \bibinfo{year}{2017}\natexlab{}.
\newblock \showarticletitle{Data decisions and theoretical implications when
  adversarially learning fair representations}.
\newblock \bibinfo{journal}{\emph{arXiv preprint arXiv:1707.00075}}
  (\bibinfo{year}{2017}).
\newblock


\bibitem[\protect\citeauthoryear{Calders, Kamiran, and Pechenizkiy}{Calders
  et~al\mbox{.}}{2009}]%
        {calders2009building}
\bibfield{author}{\bibinfo{person}{Toon Calders}, \bibinfo{person}{Faisal
  Kamiran}, {and} \bibinfo{person}{Mykola Pechenizkiy}.}
  \bibinfo{year}{2009}\natexlab{}.
\newblock \showarticletitle{Building classifiers with independency
  constraints}. In \bibinfo{booktitle}{\emph{2009 IEEE International Conference
  on Data Mining Workshops}}. IEEE, \bibinfo{pages}{13--18}.
\newblock


\bibitem[\protect\citeauthoryear{Calders and Verwer}{Calders and
  Verwer}{2010}]%
        {calders2010three}
\bibfield{author}{\bibinfo{person}{Toon Calders} {and} \bibinfo{person}{Sicco
  Verwer}.} \bibinfo{year}{2010}\natexlab{}.
\newblock \showarticletitle{Three naive Bayes approaches for
  discrimination-free classification}.
\newblock \bibinfo{journal}{\emph{Data Mining and Knowledge Discovery}}
  \bibinfo{volume}{21}, \bibinfo{number}{2} (\bibinfo{year}{2010}),
  \bibinfo{pages}{277--292}.
\newblock


\bibitem[\protect\citeauthoryear{Celis, Straszak, and Vishnoi}{Celis
  et~al\mbox{.}}{2017}]%
        {celis2017ranking}
\bibfield{author}{\bibinfo{person}{L~Elisa Celis}, \bibinfo{person}{Damian
  Straszak}, {and} \bibinfo{person}{Nisheeth~K Vishnoi}.}
  \bibinfo{year}{2017}\natexlab{}.
\newblock \showarticletitle{Ranking with fairness constraints}.
\newblock \bibinfo{journal}{\emph{arXiv preprint arXiv:1704.06840}}
  (\bibinfo{year}{2017}).
\newblock


\bibitem[\protect\citeauthoryear{Chouldechova}{Chouldechova}{2017}]%
        {chouldechova2017fair}
\bibfield{author}{\bibinfo{person}{Alexandra Chouldechova}.}
  \bibinfo{year}{2017}\natexlab{}.
\newblock \showarticletitle{Fair prediction with disparate impact: A study of
  bias in recidivism prediction instruments}.
\newblock \bibinfo{journal}{\emph{Big data}} \bibinfo{volume}{5},
  \bibinfo{number}{2} (\bibinfo{year}{2017}), \bibinfo{pages}{153--163}.
\newblock


\bibitem[\protect\citeauthoryear{Cotter, Jiang, and Sridharan}{Cotter
  et~al\mbox{.}}{2019}]%
        {cotter2019two}
\bibfield{author}{\bibinfo{person}{Andrew Cotter}, \bibinfo{person}{Heinrich
  Jiang}, {and} \bibinfo{person}{Karthik Sridharan}.}
  \bibinfo{year}{2019}\natexlab{}.
\newblock \showarticletitle{Two-player games for efficient non-convex
  constrained optimization}. In \bibinfo{booktitle}{\emph{Algorithmic Learning
  Theory}}. PMLR, \bibinfo{pages}{300--332}.
\newblock


\bibitem[\protect\citeauthoryear{Dixon, Li, Sorensen, Thain, and
  Vasserman}{Dixon et~al\mbox{.}}{2018}]%
        {dixon2018measuring}
\bibfield{author}{\bibinfo{person}{Lucas Dixon}, \bibinfo{person}{John Li},
  \bibinfo{person}{Jeffrey Sorensen}, \bibinfo{person}{Nithum Thain}, {and}
  \bibinfo{person}{Lucy Vasserman}.} \bibinfo{year}{2018}\natexlab{}.
\newblock \showarticletitle{Measuring and mitigating unintended bias in text
  classification}. In \bibinfo{booktitle}{\emph{Proceedings of the 2018
  AAAI/ACM Conference on AI, Ethics, and Society}}. \bibinfo{pages}{67--73}.
\newblock


\bibitem[\protect\citeauthoryear{Feldman, Friedler, Moeller, Scheidegger, and
  Venkatasubramanian}{Feldman et~al\mbox{.}}{2015}]%
        {feldman2015certifying}
\bibfield{author}{\bibinfo{person}{Michael Feldman}, \bibinfo{person}{Sorelle~A
  Friedler}, \bibinfo{person}{John Moeller}, \bibinfo{person}{Carlos
  Scheidegger}, {and} \bibinfo{person}{Suresh Venkatasubramanian}.}
  \bibinfo{year}{2015}\natexlab{}.
\newblock \showarticletitle{Certifying and removing disparate impact}. In
  \bibinfo{booktitle}{\emph{proceedings of the 21th ACM SIGKDD international
  conference on knowledge discovery and data mining}}.
  \bibinfo{pages}{259--268}.
\newblock


\bibitem[\protect\citeauthoryear{Freund, Iyer, Schapire, and Singer}{Freund
  et~al\mbox{.}}{2003}]%
        {freund2003efficient}
\bibfield{author}{\bibinfo{person}{Yoav Freund}, \bibinfo{person}{Raj Iyer},
  \bibinfo{person}{Robert~E Schapire}, {and} \bibinfo{person}{Yoram Singer}.}
  \bibinfo{year}{2003}\natexlab{}.
\newblock \showarticletitle{An efficient boosting algorithm for combining
  preferences}.
\newblock \bibinfo{journal}{\emph{Journal of machine learning research}}
  \bibinfo{volume}{4}, \bibinfo{number}{Nov} (\bibinfo{year}{2003}),
  \bibinfo{pages}{933--969}.
\newblock


\bibitem[\protect\citeauthoryear{Friedler, Scheidegger, Venkatasubramanian,
  Choudhary, Hamilton, and Roth}{Friedler et~al\mbox{.}}{2019}]%
        {friedler2019comparative}
\bibfield{author}{\bibinfo{person}{Sorelle~A Friedler}, \bibinfo{person}{Carlos
  Scheidegger}, \bibinfo{person}{Suresh Venkatasubramanian},
  \bibinfo{person}{Sonam Choudhary}, \bibinfo{person}{Evan~P Hamilton}, {and}
  \bibinfo{person}{Derek Roth}.} \bibinfo{year}{2019}\natexlab{}.
\newblock \showarticletitle{A comparative study of fairness-enhancing
  interventions in machine learning}. In \bibinfo{booktitle}{\emph{Proceedings
  of the Conference on Fairness, Accountability, and Transparency}}.
  \bibinfo{pages}{329--338}.
\newblock


\bibitem[\protect\citeauthoryear{Geyik, Ambler, and Kenthapadi}{Geyik
  et~al\mbox{.}}{2019}]%
        {geyik2019fairness}
\bibfield{author}{\bibinfo{person}{Sahin~Cem Geyik}, \bibinfo{person}{Stuart
  Ambler}, {and} \bibinfo{person}{Krishnaram Kenthapadi}.}
  \bibinfo{year}{2019}\natexlab{}.
\newblock \showarticletitle{Fairness-aware ranking in search \& recommendation
  systems with application to LinkedIn talent search}. In
  \bibinfo{booktitle}{\emph{Proceedings of the 25th ACM SIGKDD International
  Conference on Knowledge Discovery \& Data Mining}}.
  \bibinfo{pages}{2221--2231}.
\newblock


\bibitem[\protect\citeauthoryear{Hanley and McNeil}{Hanley and McNeil}{1982}]%
        {hanley1982meaning}
\bibfield{author}{\bibinfo{person}{James~A Hanley} {and}
  \bibinfo{person}{Barbara~J McNeil}.} \bibinfo{year}{1982}\natexlab{}.
\newblock \showarticletitle{The meaning and use of the area under a receiver
  operating characteristic (ROC) curve.}
\newblock \bibinfo{journal}{\emph{Radiology}} \bibinfo{volume}{143},
  \bibinfo{number}{1} (\bibinfo{year}{1982}), \bibinfo{pages}{29--36}.
\newblock


\bibitem[\protect\citeauthoryear{Hardt, Price, and Srebro}{Hardt
  et~al\mbox{.}}{2016}]%
        {hardt2016equality}
\bibfield{author}{\bibinfo{person}{Moritz Hardt}, \bibinfo{person}{Eric Price},
  {and} \bibinfo{person}{Nati Srebro}.} \bibinfo{year}{2016}\natexlab{}.
\newblock \showarticletitle{Equality of opportunity in supervised learning}. In
  \bibinfo{booktitle}{\emph{Advances in neural information processing
  systems}}. \bibinfo{pages}{3315--3323}.
\newblock


\bibitem[\protect\citeauthoryear{Harutyunyan, Khachatrian, Kale, Ver~Steeg, and
  Galstyan}{Harutyunyan et~al\mbox{.}}{2019}]%
        {harutyunyan2019multitask}
\bibfield{author}{\bibinfo{person}{Hrayr Harutyunyan}, \bibinfo{person}{Hrant
  Khachatrian}, \bibinfo{person}{David~C Kale}, \bibinfo{person}{Greg
  Ver~Steeg}, {and} \bibinfo{person}{Aram Galstyan}.}
  \bibinfo{year}{2019}\natexlab{}.
\newblock \showarticletitle{Multitask learning and benchmarking with clinical
  time series data}.
\newblock \bibinfo{journal}{\emph{Scientific data}} \bibinfo{volume}{6},
  \bibinfo{number}{1} (\bibinfo{year}{2019}), \bibinfo{pages}{1--18}.
\newblock


\bibitem[\protect\citeauthoryear{Johnson, Pollard, Shen, Li-wei, Feng,
  Ghassemi, Moody, Szolovits, Celi, and Mark}{Johnson et~al\mbox{.}}{2016}]%
        {johnson2016mimic}
\bibfield{author}{\bibinfo{person}{Alistair~EW Johnson}, \bibinfo{person}{Tom~J
  Pollard}, \bibinfo{person}{Lu Shen}, \bibinfo{person}{H~Lehman Li-wei},
  \bibinfo{person}{Mengling Feng}, \bibinfo{person}{Mohammad Ghassemi},
  \bibinfo{person}{Benjamin Moody}, \bibinfo{person}{Peter Szolovits},
  \bibinfo{person}{Leo~Anthony Celi}, {and} \bibinfo{person}{Roger~G Mark}.}
  \bibinfo{year}{2016}\natexlab{}.
\newblock \showarticletitle{MIMIC-III, a freely accessible critical care
  database}.
\newblock \bibinfo{journal}{\emph{Scientific data}}  \bibinfo{volume}{3}
  (\bibinfo{year}{2016}), \bibinfo{pages}{160035}.
\newblock


\bibitem[\protect\citeauthoryear{Kallus and Zhou}{Kallus and Zhou}{2018}]%
        {kallus2018residual}
\bibfield{author}{\bibinfo{person}{Nathan Kallus} {and} \bibinfo{person}{Angela
  Zhou}.} \bibinfo{year}{2018}\natexlab{}.
\newblock \showarticletitle{Residual Unfairness in Fair Machine Learning from
  Prejudiced Data}. In \bibinfo{booktitle}{\emph{International Conference on
  Machine Learning}}. \bibinfo{pages}{2439--2448}.
\newblock


\bibitem[\protect\citeauthoryear{Kallus and Zhou}{Kallus and Zhou}{2019}]%
        {kallus2019fairness}
\bibfield{author}{\bibinfo{person}{Nathan Kallus} {and} \bibinfo{person}{Angela
  Zhou}.} \bibinfo{year}{2019}\natexlab{}.
\newblock \showarticletitle{The fairness of risk scores beyond classification:
  Bipartite ranking and the xauc metric}. In \bibinfo{booktitle}{\emph{Advances
  in Neural Information Processing Systems}}. \bibinfo{pages}{3433--3443}.
\newblock


\bibitem[\protect\citeauthoryear{Kamishima, Akaho, and Sakuma}{Kamishima
  et~al\mbox{.}}{2011}]%
        {kamishima2011fairness}
\bibfield{author}{\bibinfo{person}{Toshihiro Kamishima},
  \bibinfo{person}{Shotaro Akaho}, {and} \bibinfo{person}{Jun Sakuma}.}
  \bibinfo{year}{2011}\natexlab{}.
\newblock \showarticletitle{Fairness-aware learning through regularization
  approach}. In \bibinfo{booktitle}{\emph{2011 IEEE 11th International
  Conference on Data Mining Workshops}}. IEEE, \bibinfo{pages}{643--650}.
\newblock


\bibitem[\protect\citeauthoryear{Kleinberg, Mullainathan, and
  Raghavan}{Kleinberg et~al\mbox{.}}{2016}]%
        {kleinberg2016inherent}
\bibfield{author}{\bibinfo{person}{Jon Kleinberg}, \bibinfo{person}{Sendhil
  Mullainathan}, {and} \bibinfo{person}{Manish Raghavan}.}
  \bibinfo{year}{2016}\natexlab{}.
\newblock \showarticletitle{Inherent trade-offs in the fair determination of
  risk scores}.
\newblock \bibinfo{journal}{\emph{arXiv preprint arXiv:1609.05807}}
  (\bibinfo{year}{2016}).
\newblock


\bibitem[\protect\citeauthoryear{Kohavi}{Kohavi}{1996}]%
        {kohavi1996scaling}
\bibfield{author}{\bibinfo{person}{Ron Kohavi}.}
  \bibinfo{year}{1996}\natexlab{}.
\newblock \showarticletitle{Scaling up the accuracy of naive-bayes classifiers:
  A decision-tree hybrid.}. In \bibinfo{booktitle}{\emph{Kdd}},
  Vol.~\bibinfo{volume}{96}. \bibinfo{pages}{202--207}.
\newblock


\bibitem[\protect\citeauthoryear{Levy}{Levy}{1999}]%
        {levy199950}
\bibfield{author}{\bibinfo{person}{Daniel Levy}.}
  \bibinfo{year}{1999}\natexlab{}.
\newblock \bibinfo{booktitle}{\emph{50 years of discovery: medical milestones
  from the National Heart, Lung, and Blood Institute's Framingham Heart
  Study}}.
\newblock \bibinfo{publisher}{Center for Bio-Medical Communication, Inc.}
\newblock


\bibitem[\protect\citeauthoryear{Lichtenstein, Fischhoff, and
  Phillips}{Lichtenstein et~al\mbox{.}}{1981}]%
        {lichtenstein1981calibration}
\bibfield{author}{\bibinfo{person}{Sarah Lichtenstein}, \bibinfo{person}{Baruch
  Fischhoff}, {and} \bibinfo{person}{Lawrence~D Phillips}.}
  \bibinfo{year}{1981}\natexlab{}.
\newblock \bibinfo{booktitle}{\emph{Calibration of probabilities: The state of
  the art to 1980}}.
\newblock \bibinfo{type}{{T}echnical {R}eport}. \bibinfo{institution}{DECISION
  RESEARCH EUGENE OR}.
\newblock


\bibitem[\protect\citeauthoryear{Louizos, Swersky, Li, Welling, and
  Zemel}{Louizos et~al\mbox{.}}{2015}]%
        {louizos2015variational}
\bibfield{author}{\bibinfo{person}{Christos Louizos}, \bibinfo{person}{Kevin
  Swersky}, \bibinfo{person}{Yujia Li}, \bibinfo{person}{Max Welling}, {and}
  \bibinfo{person}{Richard Zemel}.} \bibinfo{year}{2015}\natexlab{}.
\newblock \showarticletitle{The variational fair autoencoder}.
\newblock \bibinfo{journal}{\emph{arXiv preprint arXiv:1511.00830}}
  (\bibinfo{year}{2015}).
\newblock


\bibitem[\protect\citeauthoryear{Madras, Creager, Pitassi, and Zemel}{Madras
  et~al\mbox{.}}{2018}]%
        {madras2018learning}
\bibfield{author}{\bibinfo{person}{David Madras}, \bibinfo{person}{Elliot
  Creager}, \bibinfo{person}{Toniann Pitassi}, {and} \bibinfo{person}{Richard
  Zemel}.} \bibinfo{year}{2018}\natexlab{}.
\newblock \showarticletitle{Learning adversarially fair and transferable
  representations}.
\newblock \bibinfo{journal}{\emph{arXiv preprint arXiv:1802.06309}}
  (\bibinfo{year}{2018}).
\newblock


\bibitem[\protect\citeauthoryear{Menon and Williamson}{Menon and
  Williamson}{2016}]%
        {menon2016bipartite}
\bibfield{author}{\bibinfo{person}{Aditya~Krishna Menon} {and}
  \bibinfo{person}{Robert~C Williamson}.} \bibinfo{year}{2016}\natexlab{}.
\newblock \showarticletitle{Bipartite ranking: a risk-theoretic perspective}.
\newblock \bibinfo{journal}{\emph{The Journal of Machine Learning Research}}
  \bibinfo{volume}{17}, \bibinfo{number}{1} (\bibinfo{year}{2016}),
  \bibinfo{pages}{6766--6867}.
\newblock


\bibitem[\protect\citeauthoryear{Moylan, Brady, Johnson, Smith, Tuttle-Newhall,
  and Muir}{Moylan et~al\mbox{.}}{2008}]%
        {moylan2008disparities}
\bibfield{author}{\bibinfo{person}{Cynthia~A Moylan}, \bibinfo{person}{Carla~W
  Brady}, \bibinfo{person}{Jeffrey~L Johnson}, \bibinfo{person}{Alastair~D
  Smith}, \bibinfo{person}{Janet~E Tuttle-Newhall}, {and}
  \bibinfo{person}{Andrew~J Muir}.} \bibinfo{year}{2008}\natexlab{}.
\newblock \showarticletitle{Disparities in liver transplantation before and
  after introduction of the MELD score}.
\newblock \bibinfo{journal}{\emph{{JAMA}}} \bibinfo{volume}{300},
  \bibinfo{number}{20} (\bibinfo{year}{2008}), \bibinfo{pages}{2371--2378}.
\newblock


\bibitem[\protect\citeauthoryear{Narasimhan and Agarwal}{Narasimhan and
  Agarwal}{2013}]%
        {narasimhan2013relationship}
\bibfield{author}{\bibinfo{person}{Harikrishna Narasimhan} {and}
  \bibinfo{person}{Shivani Agarwal}.} \bibinfo{year}{2013}\natexlab{}.
\newblock \showarticletitle{On the relationship between binary classification,
  bipartite ranking, and binary class probability estimation}. In
  \bibinfo{booktitle}{\emph{Advances in Neural Information Processing
  Systems}}. \bibinfo{pages}{2913--2921}.
\newblock


\bibitem[\protect\citeauthoryear{Narasimhan, Cotter, Gupta, and
  Wang}{Narasimhan et~al\mbox{.}}{2020}]%
        {narasimhan2020pairwise}
\bibfield{author}{\bibinfo{person}{Harikrishna Narasimhan},
  \bibinfo{person}{Andrew Cotter}, \bibinfo{person}{Maya~R Gupta}, {and}
  \bibinfo{person}{Serena Wang}.} \bibinfo{year}{2020}\natexlab{}.
\newblock \showarticletitle{Pairwise Fairness for Ranking and Regression.}. In
  \bibinfo{booktitle}{\emph{AAAI}}. \bibinfo{pages}{5248--5255}.
\newblock


\bibitem[\protect\citeauthoryear{Pollard, Johnson, Raffa, Celi, Mark, and
  Badawi}{Pollard et~al\mbox{.}}{2018}]%
        {pollard2018eicu}
\bibfield{author}{\bibinfo{person}{Tom~J Pollard}, \bibinfo{person}{Alistair~EW
  Johnson}, \bibinfo{person}{Jesse~D Raffa}, \bibinfo{person}{Leo~A Celi},
  \bibinfo{person}{Roger~G Mark}, {and} \bibinfo{person}{Omar Badawi}.}
  \bibinfo{year}{2018}\natexlab{}.
\newblock \showarticletitle{The eICU Collaborative Research Database, a freely
  available multi-center database for critical care research}.
\newblock \bibinfo{journal}{\emph{Scientific data}}  \bibinfo{volume}{5}
  (\bibinfo{year}{2018}), \bibinfo{pages}{180178}.
\newblock


\bibitem[\protect\citeauthoryear{Singh and Joachims}{Singh and
  Joachims}{2018}]%
        {singh2018fairness}
\bibfield{author}{\bibinfo{person}{Ashudeep Singh} {and}
  \bibinfo{person}{Thorsten Joachims}.} \bibinfo{year}{2018}\natexlab{}.
\newblock \showarticletitle{Fairness of exposure in rankings}. In
  \bibinfo{booktitle}{\emph{Proceedings of the 24th ACM SIGKDD International
  Conference on Knowledge Discovery \& Data Mining}}.
  \bibinfo{pages}{2219--2228}.
\newblock


\bibitem[\protect\citeauthoryear{Singh and Joachims}{Singh and
  Joachims}{2019}]%
        {singh2019policy}
\bibfield{author}{\bibinfo{person}{Ashudeep Singh} {and}
  \bibinfo{person}{Thorsten Joachims}.} \bibinfo{year}{2019}\natexlab{}.
\newblock \showarticletitle{Policy learning for fairness in ranking}. In
  \bibinfo{booktitle}{\emph{Advances in Neural Information Processing
  Systems}}. \bibinfo{pages}{5427--5437}.
\newblock


\bibitem[\protect\citeauthoryear{Vogel, Bellet, and Cl{\'e}men{\c{c}}on}{Vogel
  et~al\mbox{.}}{2020}]%
        {vogel2020learning}
\bibfield{author}{\bibinfo{person}{Robin Vogel}, \bibinfo{person}{Aur{\'e}lien
  Bellet}, {and} \bibinfo{person}{St{\'e}phan Cl{\'e}men{\c{c}}on}.}
  \bibinfo{year}{2020}\natexlab{}.
\newblock \showarticletitle{Learning Fair Scoring Functions: Fairness
  Definitions, Algorithms and Generalization Bounds for Bipartite Ranking}.
\newblock \bibinfo{journal}{\emph{arXiv preprint arXiv:2002.08159}}
  (\bibinfo{year}{2020}).
\newblock


\bibitem[\protect\citeauthoryear{Wiesner, Edwards, Freeman, Harper, Kim,
  Kamath, Kremers, Lake, Howard, Merion, et~al\mbox{.}}{Wiesner
  et~al\mbox{.}}{2003}]%
        {wiesner2003model}
\bibfield{author}{\bibinfo{person}{Russell Wiesner}, \bibinfo{person}{Erick
  Edwards}, \bibinfo{person}{Richard Freeman}, \bibinfo{person}{Ann Harper},
  \bibinfo{person}{Ray Kim}, \bibinfo{person}{Patrick Kamath},
  \bibinfo{person}{Walter Kremers}, \bibinfo{person}{John Lake},
  \bibinfo{person}{Todd Howard}, \bibinfo{person}{Robert~M Merion},
  {et~al\mbox{.}}} \bibinfo{year}{2003}\natexlab{}.
\newblock \showarticletitle{Model for end-stage liver disease (MELD) and
  allocation of donor livers}.
\newblock \bibinfo{journal}{\emph{Gastroenterology}} \bibinfo{volume}{124},
  \bibinfo{number}{1} (\bibinfo{year}{2003}), \bibinfo{pages}{91--96}.
\newblock


\bibitem[\protect\citeauthoryear{Yang and Stoyanovich}{Yang and
  Stoyanovich}{2017}]%
        {yang2017measuring}
\bibfield{author}{\bibinfo{person}{Ke Yang} {and} \bibinfo{person}{Julia
  Stoyanovich}.} \bibinfo{year}{2017}\natexlab{}.
\newblock \showarticletitle{Measuring fairness in ranked outputs}. In
  \bibinfo{booktitle}{\emph{Proceedings of the 29th International Conference on
  Scientific and Statistical Database Management}}. \bibinfo{pages}{1--6}.
\newblock


\bibitem[\protect\citeauthoryear{Zafar, Valera, Gomez~Rodriguez, and
  Gummadi}{Zafar et~al\mbox{.}}{2017}]%
        {zafar2017fairness}
\bibfield{author}{\bibinfo{person}{Muhammad~Bilal Zafar},
  \bibinfo{person}{Isabel Valera}, \bibinfo{person}{Manuel Gomez~Rodriguez},
  {and} \bibinfo{person}{Krishna~P Gummadi}.} \bibinfo{year}{2017}\natexlab{}.
\newblock \showarticletitle{Fairness beyond disparate treatment \& disparate
  impact: Learning classification without disparate mistreatment}. In
  \bibinfo{booktitle}{\emph{Proceedings of the 26th international conference on
  world wide web}}. \bibinfo{pages}{1171--1180}.
\newblock


\bibitem[\protect\citeauthoryear{Zafar, Valera, Rodriguez, and Gummadi}{Zafar
  et~al\mbox{.}}{2015}]%
        {zafar2015fairness}
\bibfield{author}{\bibinfo{person}{Muhammad~Bilal Zafar},
  \bibinfo{person}{Isabel Valera}, \bibinfo{person}{Manuel~Gomez Rodriguez},
  {and} \bibinfo{person}{Krishna~P Gummadi}.} \bibinfo{year}{2015}\natexlab{}.
\newblock \showarticletitle{Fairness constraints: Mechanisms for fair
  classification}.
\newblock \bibinfo{journal}{\emph{arXiv preprint arXiv:1507.05259}}
  (\bibinfo{year}{2015}).
\newblock


\bibitem[\protect\citeauthoryear{Zehlike, Bonchi, Castillo, Hajian, Megahed,
  and Baeza-Yates}{Zehlike et~al\mbox{.}}{2017}]%
        {zehlike2017fa}
\bibfield{author}{\bibinfo{person}{Meike Zehlike}, \bibinfo{person}{Francesco
  Bonchi}, \bibinfo{person}{Carlos Castillo}, \bibinfo{person}{Sara Hajian},
  \bibinfo{person}{Mohamed Megahed}, {and} \bibinfo{person}{Ricardo
  Baeza-Yates}.} \bibinfo{year}{2017}\natexlab{}.
\newblock \showarticletitle{Fa* ir: A fair top-k ranking algorithm}. In
  \bibinfo{booktitle}{\emph{Proceedings of the 2017 ACM on Conference on
  Information and Knowledge Management}}. \bibinfo{pages}{1569--1578}.
\newblock


\bibitem[\protect\citeauthoryear{Zemel, Wu, Swersky, Pitassi, and Dwork}{Zemel
  et~al\mbox{.}}{2013}]%
        {zemel2013learning}
\bibfield{author}{\bibinfo{person}{Rich Zemel}, \bibinfo{person}{Yu Wu},
  \bibinfo{person}{Kevin Swersky}, \bibinfo{person}{Toni Pitassi}, {and}
  \bibinfo{person}{Cynthia Dwork}.} \bibinfo{year}{2013}\natexlab{}.
\newblock \showarticletitle{Learning fair representations}. In
  \bibinfo{booktitle}{\emph{International Conference on Machine Learning}}.
  \bibinfo{pages}{325--333}.
\newblock


\bibitem[\protect\citeauthoryear{Zhang, Lemoine, and Mitchell}{Zhang
  et~al\mbox{.}}{2018}]%
        {zhang2018mitigating}
\bibfield{author}{\bibinfo{person}{Brian~Hu Zhang}, \bibinfo{person}{Blake
  Lemoine}, {and} \bibinfo{person}{Margaret Mitchell}.}
  \bibinfo{year}{2018}\natexlab{}.
\newblock \showarticletitle{Mitigating unwanted biases with adversarial
  learning}. In \bibinfo{booktitle}{\emph{Proceedings of the 2018 AAAI/ACM
  Conference on AI, Ethics, and Society}}. \bibinfo{pages}{335--340}.
\newblock


\end{thebibliography}

\clearpage
\appendix

\onecolumn
\section{Analysis on AUC, xAUC and PRF}
\subsection{Decomposing AUC into xAUC and iAUC}
In fact, the AUC of the risk function can be decomposed into xAUC and iAUC, while iAUC means the probability of positive instances rank above negative instances in the same group:

\begin{equation}
\begin{aligned}
&\mathrm{AUC}=\frac{1}{k} \cdot \left(k_{a} \cdot  \mathrm{iAUC}(a)+k_{a, b} \cdot \mathrm{xAUC}(a,b)+k_{b, a} \cdot \mathrm{xAUC}(b,a)+k_{b} \cdot \mathrm{iAUC}(b)\right) \\
& \mathrm{iAUC}(a) = \operatorname{Pr}\left[\operatorname{S}_{1}^{a}>\operatorname{S}_{0}^{a}\right] = \frac{1}{k_{a}} \sum\nolimits_{i: i \in a, \operatorname{Y}_{i}=1} \sum\nolimits_{j: j \in a, \operatorname{Y}_{j}=0} \mathbb{I}\left[R\left(\operatorname{X}_{i}, a\right)>R\left(\operatorname{X}_{j}, a\right)\right]\\
& \mathrm{iAUC}(b) = \operatorname{Pr}\left[\operatorname{S}_{1}^{b}>\operatorname{S}_{0}^{b}\right] = \frac{1}{k_{b}} \sum\nolimits_{i: i \in b, \operatorname{Y}_{i}=1} \sum\nolimits_{j: j \in b, \operatorname{Y}_{j}=0} \mathbb{I}\left[R\left(\operatorname{X}_{i}, b\right)>R\left(\operatorname{X}_{j}, b\right)\right]\\
& \mathrm{xAUC}(a, b)=\operatorname{Pr}\left[\operatorname{S}_{1}^{a}>\operatorname{S}_{0}^{b}\right] = \frac{1}{k_{a,b}} \cdot \sum\nolimits_{i:i\in a,\operatorname{Y}_{i}=1}\sum\nolimits_{j: j\in b, \operatorname{Y}_{j}=0}\mathbb{I} \left[R(\operatorname{X}_{i},a) > R(\operatorname{X}_{j},b)\right]\\
& \mathrm{xAUC}(b,a)=\operatorname{Pr}\left[\operatorname{S}_{1}^{b}>\operatorname{S}_{0}^{a}\right] = \frac{1}{k_{b,a}} \cdot \sum\nolimits_{i:i\in b,\operatorname{Y}_{i}=1}\sum\nolimits_{j: j\in a, \operatorname{Y}_{j}=0}\mathbb{I} \left[R(\operatorname{X}_{i},b) > R(\operatorname{X}_{j},a)\right]
\end{aligned}
\label{eq:AUC_decom}
\end{equation}

in which $k_{a,b}$, $k_{b,a}$ are the same as in the main text while $k = n_{0}n_{1}$, $k_{a} = n^{a}_{0}n^{a}_{1}$, $k_{b} = n^{b}_{0}n^{b}_{1}$.

\subsection{Analysis on PRF}
From the decomposition Eq.~\ref{eq:AUC_decom}, the metric of PRF can be decomposed into xAUC and iAUC as follows:

\begin{equation}
\begin{aligned}
& \mathrm{PRF}(a) = \operatorname{Pr}[\operatorname{S}_{1}^{a}>\operatorname{S}_{0}] = \frac{1}{n^{a}_{1}n_{0}} \cdot \sum\nolimits_{i:i\in a,\operatorname{Y}_{i}=1}\sum\nolimits_{j:\operatorname{Y}_{j}=0}\mathbb{I} \left[R(\operatorname{X}_{i},a) > R(\operatorname{X}_{j})\right] \\
& \mathrm{PRF}(b) = \operatorname{Pr}[\operatorname{S}_{1}^{b}>\operatorname{S}_{0}] = \frac{1}{n^{b}_{1}n_{0}} \cdot \sum\nolimits_{i:i\in b,\operatorname{Y}_{i}=1}\sum\nolimits_{j:\operatorname{Y}_{j}=0}\mathbb{I} \left[R(\operatorname{X}_{i},a) > R(\operatorname{X}_{j})\right]
\end{aligned}
\end{equation}

According to the Eq.~\ref{eq:AUC_decom}, the probability $\operatorname{Pr}[S_{1}^{a}>S_{0}]$, $\operatorname{Pr}[S_{1}^{b}>S_{0}]$ can be rewritten as follows:

\begin{equation}
\begin{aligned}
& \operatorname{Pr}[\mathrm{S}_{1}^{a}>\mathrm{S}_{0}] =\frac{n^{b}_{0}}{n_{0}} \cdot \mathrm{xAUC}(a,b) + \frac{n^{a}_{0}}{n_{0}} \cdot \operatorname{iAUC}(a) \\
& \operatorname{Pr}[\mathrm{S}_{1}^{b}>\mathrm{S}_{0}] = \frac{n^{a}_{0}}{n_{0}} \cdot \mathrm{xAUC}(b,a) + \frac{n^{b}_{0}}{n_{0}} \cdot \operatorname{iAUC}(b),
\end{aligned}
\end{equation}

\section{Proof of Proposition 1}
\subsection{($\Delta\mathrm{xAUC}$)}
\label{prop:1}
\begin{proof}
Denote $T^{a}_{1}(i) = \frac{\sum_{k \leq i} \mathbb{I} \left[\operatorname{Y}_{p^{a(k)}} = 1\right]}{n_1^a}$, $T^{a}_{1}(i) $ is monotonically increasing on $i$ from 0 to 1. Reversely, denote $\overline{T^{a}_{0}}(i) = \frac{\sum_{k \textgreater i} \mathbb{I} \left[\operatorname{Y}_{\operatorname{p}^{a(k)}} = 0\right]}{n_0^a}$ and $\overline{T^{a}_{0}}(i)$ is monotonically decreasing on $i$ from 1 to 0. Then $T^{a}_{1}(i) - \overline{T^{a}_{0}}(i)$ is monotonically increasing on $i$ from -1 to 1. So there exists an $i'$ that $T^{a}_{1}(i') - \overline{T^{a}_{0}}(i') < 0$ and $T^{a}_{1}(i' + 1) - \overline{T^{a}_{0}}(i' + 1) \geq 0$. Since the
increment of $T^{a}_{1}(i) - \overline{T^{a}_{0}}(i)$ when $i$ increases by 1 satisfies $\Delta{(T^{a}_{1}(i) - \overline{T^{a}_{0}}(i))} \leq \max(\frac{1}{n_1^a}, \frac{1}{n_0^a})$, we can get $-\max(\frac{1}{n_1^a}, \frac{1}{n_0^a}) \leq T^{a}_{1}(i') - \overline{T^{a}_{0}}(i') < 0$.

Consider a cross-group ordering which is generated by inserting the whole sequence $\operatorname{p}^{b}$ between $\operatorname{p}^{a(i')}$ and $\operatorname{p}^{a(i'+1)}$, this operation will result in that positive examples $\operatorname{p}^{a(k)}$ with $\mathrm{Y}_{\operatorname{p}^{a(k)}} = 1, k \leq i'$ will be ranked higher than all the negative examples in $\operatorname{p}^{b}$. And positive examples $\operatorname{p}^{a(k)}$ with $\mathrm{Y}_{\operatorname{p}^{a(k)}} = 1, k > i'$ will be ranked lower than all the negative examples in $\operatorname{p}^{b}$. Then $\mathrm{xAUC}(a,b)$ equals $T^{a}_{1}(i')$ with this cross-group ordering. Similarly, we can obtain $\mathrm{xAUC}(b,a) = \overline{T^{a}_{0}}(i')$. Then $\Delta \mathrm{xAUC} = |\overline{T^{a}_{0}}(i') - T^{a}_{1}(i')|$. According to the discussion above, $\Delta\mathrm{xAUC} \leq \max(\frac{1}{n_1^a}, \frac{1}{n_0^a})$.

If we consider the cross-group ordering generated by inserting the whole sequence $\operatorname{p}^{a}$ between $\operatorname{p}^{b(j)}$ and $\operatorname{p}^{b(j+1)}$, symmetrically  we will find that there exists a corresponding $j'$ that $\Delta\mathrm{xAUC} \leq \max(\frac{1}{n_1^b}, \frac{1}{n_0^b})$  with this cross-group ordering. We can choose  one of these two cross-group ordering operations to achieve $\Delta\mathrm{xAUC} \leq \min(\max(\frac{1}{n_1^b}, \frac{1}{n_0^b}), \max(\frac{1}{n_1^a}, \frac{1}{n_0^a}))$
\end{proof}

\subsection{($\Delta\mathrm{PRF})$}
\begin{proof}
Denote $\mathrm{C} = \frac{n_0^b}{n_0}\operatorname{iAUC(b)} - \frac{n_0^a}{n_0}\operatorname{iAUC(a)}$, we can get $-\frac{n_0^a}{n_0} \leq \mathrm{C} \leq \frac{n_0^b}{n_0}$. Since changing cross-group ordering does not affect inner-group ordering, $\mathrm{C}$ is constant for given $\operatorname{p}^{b}$ and $\operatorname{p}^{a}$.

With the same definition of $T^{a}_{1}(i)$ and $\overline{T^{a}_{0}}(i)$ in the Section \ref{prop:1}, we will have $\frac{n_0^b}{n_0}T^{a}_{1}(i) - \frac{n_0^a}{n_0}\overline{T^{a}_{0}}(i)$ is monotonically increasing on $i$ from $-\frac{n_0^a}{n_0}$ to $\frac{n_0^b}{n_0}$. Since $0 \in [-\frac{n_0^a}{n_0} - \mathrm{C}, \frac{n_0^b}{n_0} - \mathrm{C}]$, there exists $i'$ satisfying $\frac{n_0^b}{n_0}T^{a}_{1}(i') - \frac{n_0^a}{n_0}\overline{T^{a}_{0}}(i') - \mathrm{C} < 0$, $\frac{n_0^b}{n_0}T^{a}_{1}(i' + 1) - \frac{n_0^a}{n_0}\overline{T^{a}_{0}}(i' + 1) - \mathrm{C} \geq 0$ or $\frac{n_0^b}{n_0}T^{a}_{1}(i') - \frac{n_0^a}{n_0}\overline{T^{a}_{0}}(i') - \mathrm{C} \leq 0$, $\frac{n_0^b}{n_0}T^{a}_{1}(i' + 1) - \frac{n_0^a}{n_0}\overline{T^{a}_{0}}(i' + 1) - \mathrm{C} > 0$. Since the increment of $\frac{n_0^b}{n_0}T^{a}_{1}(i) - \frac{n_0^a}{n_0}\overline{T^{a}_{0}}(i) - \mathrm{C}$ when $i$ increases by 1 satisfies $\Delta{(\frac{n_0^b}{n_0}T^{a}_{1}(i) - \frac{n_0^a}{n_0}\overline{T^{a}_{0}}(i) - \mathrm{C})} \leq \max(\frac{n^b_0}{n_1^a n_0}, \frac{1}{n_0})$, $-\max(\frac{n^b_0}{n_1^a n_0}, \frac{1}{n_0}) \leq \frac{n_0^a}{n_0}T^{a}_{1}(i') - \frac{n_0^b}{n_0}\overline{T^{a}_{0}}(i') - \mathrm{C} \leq \max(\frac{n^b_0}{n_1^a n_0}, \frac{1}{n_0})$.

Consider an cross-group ordering generated by inserting the whole sequence $\operatorname{p}^{b}$ between $\operatorname{p}^{a(i')}$ and $\operatorname{p}^{a(i'+1)}$. $\mathrm{xAUC}(a,b) = T^{a}_{1}(i')$ and $\mathrm{xAUC}(b,a) = \overline{T^{a}_{0}}(i')$. As in the calculation of $\Delta\operatorname{PRF} = |\frac{n_0^b}{n_0}\mathrm{xAUC}(a,b)- \frac{n_0^a}{n_0}\mathrm{xAUC}(b,a) - (\frac{n_0^b}{n_0}\operatorname{iAUC}(b) - \frac{n_0^a}{n_0}\operatorname{iAUC}(a)) | = |\frac{n_0^b}{n_0}T^{a}_{1}(i') - \frac{n_0^a}{n_0}\overline{T^{a}_{0}}(i') - \mathrm{C}|$, we can get $\Delta\operatorname{PRF} \leq \max(\frac{n^b_0}{n_1^a n_0}, \frac{1}{n_0})$.

If we consider the cross-group ordering generated by inserting the whole sequence $\operatorname{p}^{a}$ between $\operatorname{p}^{b(j)}$ and $\operatorname{p}^{b(j+1)}$, symmetrically we will find there exists a corresponding $j'$ that $\Delta\operatorname{PRF} \leq \max(\frac{n^a_0}{n_1^b n_0}, \frac{1}{n_0})$. We can choose one from these two cross-group ordering operations and achieve $\Delta \operatorname{PRF} \leq \min(\max(\frac{n^b_0}{n_1^a n_0}, \frac{1}{n_0}), \max(\frac{n^a_0}{n_1^b n_0}, \frac{1}{n_0}))$.
\end{proof}

\section{Analysis on Proposition 2}
\subsection{Proof of Proposition 2}
\begin{proof}
As we keep the with-in group ordering invariant, iAUC(a) and iAUC(b) remain the same after post-processing procedure. For the objective function:
\begin{equation}
\begin{aligned}
\label{eq:optimizationAUC}
J(o(\operatorname{p}^{a}, \operatorname{p}^{b}))=\text {AUC}(o(\operatorname{p}^{a}, \operatorname{p}^{b}))-\lambda \cdot\Delta\mathrm{xAUC}(o(\operatorname{p}^{a}, \operatorname{p}^{b}))
\end{aligned}
\end{equation}

the item $\text{AUC}(o(\operatorname{p}^{a}, \operatorname{p}^{b}))$ can be decomposed according to Eq.~\ref{eq:AUC_decom}. We subtract the constant part $k_{a} \cdot  \mathrm{iAUC(a)}$, $k_{b} \cdot  \mathrm{iAUC(b)}$ and multiply the formula by a constant $k$:

\begin{equation}
\begin{aligned}
J(o(\operatorname{p}^{a}, \operatorname{p}^{b})) = \frac{1}{k} (k_{a,b}  \operatorname{xAUC}(o(\operatorname{p}^{a}, \operatorname{p}^{b})) + k_{b,a}  \operatorname{xAUC}(o(\operatorname{p}^{b}, \operatorname{p}^{a})) - \lambda k  \Delta\mathrm{xAUC}(o(\operatorname{p}^{a}, \operatorname{p}^{b}))) + \mathrm{C}
\end{aligned}
\label{eq:J_decom}
\end{equation}

where $\mathrm{C}$ is constant $\mathrm{C} = \frac{k_a}{k} \cdot \mathrm{iAUC}(\operatorname{p}^{a}) + \frac{k_b}{k} \cdot \mathrm{iAUC}(\operatorname{p}^{b})$

From Eq.~\ref{eq:J_decom}, the target Eq.~\ref{eq:optimizationAUC} is equivalent to:
\begin{equation}
\begin{aligned}
\label{eq:G}
\mathrm{G}(o(\operatorname{p}^{a}, \operatorname{p}^{b}))=  k_{a,b} \cdot \mathrm{xAUC}(o(\operatorname{p}^{a}, \operatorname{p}^{b})) + k_{b,a} \cdot \mathrm{xAUC}(o(\operatorname{p}^{b}, \operatorname{p}^{a})) - \lambda \cdot k \cdot\Delta\mathrm{xAUC}(o(\operatorname{p}^{a}, \operatorname{p}^{b})).
\end{aligned}
\end{equation}
\end{proof}

Consider the definition of $\widehat{\operatorname{G}}(o(\operatorname{p}^{a(:i)}, \operatorname{p}^{b(:j)}))$ in main text Eq. (9) in which $\widehat{\operatorname{G}}(o(\operatorname{p}^{a(:i)}, \operatorname{p}^{b(:j)}))$ is induced by $\operatorname{xAUC}(o(\operatorname{p}^{a(:i)}, \operatorname{p}^{b}))$ and $\operatorname{xAUC}(o(\operatorname{p}^{b(:j)}, \operatorname{p}^{a}))$. The cross-group ordering $o(\operatorname{p}^{a(:i)}, \operatorname{p}^{b})$ means appending the sequence $\operatorname{p}^{b(j+1:n^{b})}$ to the given $o(\operatorname{p}^{a(:i)}, \operatorname{p}^{b(:j)})$. The meanings of partial $\operatorname{xAUC}(o(\operatorname{p}^{a(:i)}, \operatorname{p}^{b}))$ is as follows:

\begin{equation}
\label{eq:partxAUC}
\operatorname{xAUC}(o(\operatorname{p}^{a(:i)}, \operatorname{p}^{b})) = \frac{1}{k_{a,b}}\cdot \sum\nolimits_{k:k\leqslant i,\mathrm{Y}_{\operatorname{p}^{a(k)}}=1}\sum\nolimits_{h:h\leqslant n^{b},\mathrm{Y}_{\operatorname{p}^{a(h)}}=0}\mathbb{I} \left[\operatorname{p}^{a(k)} \succ \operatorname{p}^{b(h)}\right].
\end{equation}

\subsection{Proposition 2 and the Objective $\operatorname{\widehat{G}}$}
Considering the definition of $\widehat{\operatorname{G}}(o(\operatorname{p}^{a(:i)}, \operatorname{p}^{b(:j)}))$ in main text Eq. (9) in which $\operatorname{\widehat{G}}(o(\operatorname{p}^{a(:i)}, \operatorname{p}^{b(:j)}))$ is induced by $\operatorname{xAUC}(o(\operatorname{p}^{a(:i)}, \operatorname{p}^{b(:j)}) \textcircled{+} \operatorname{p}^{b(j+1:n^{b})})$ and $\operatorname{xAUC}(o(\operatorname{p}^{b(:j)}, \operatorname{p}^{a(:i)})  \textcircled{+} \operatorname{p}^{b(j+1:n^{b})})$, the cross-group ordering $o(\operatorname{p}^{a(:i)}, \operatorname{p}^{b(:j)}) \textcircled{+} \operatorname{p}^{b(j+1:n^{b})}$ means appending the sequence $\operatorname{p}^{b(j+1:n^{b})}$ to the given cross-group ordering $o(\operatorname{p}^{a(:i)}, \operatorname{p}^{b(:j)})$ ($o(\operatorname{p}^{b(:j)}, \operatorname{p}^{a(:i)}) \textcircled{+} \operatorname{p}^{a(i+1:n^{a})}$ is similarly defined). For expression convenience, we use $o(\operatorname{p}^{a(:i)}, \operatorname{p}^{b})$ to replace $o(\operatorname{p}^{a(:i)}, \operatorname{p}^{b(:j)}) \textcircled{+} \operatorname{p}^{b(j+1:n^{b})}$, and $o(\operatorname{p}^{a(:i+1)}, \operatorname{p}^{b})$ means $o(\operatorname{p}^{a(:i)}, \operatorname{p}^{b(:j)}) \textcircled{+} \operatorname{p}^{a(i+1)} \textcircled{+} \operatorname{p}^{b(j+1:n^{b})}$. And the property of the partial $\operatorname{xAUC}$ is as follows:

\begin{equation}
    \begin{aligned}
        & \operatorname{xAUC}(o(\operatorname{p}^{a(:i)}, \operatorname{p}^{b}) = \frac{1}{k_{a,b}}\cdot \sum\nolimits_{k:k\leqslant i,\mathrm{Y}_{\operatorname{p}^{a(k)}}=1}\sum\nolimits_{h:h\leqslant n^{b},\mathrm{Y}_{\operatorname{p}^{a(h)}}=0}\mathbb{I} \left[\operatorname{p}^{a(k)} \succ \operatorname{p}^{b(h)}\right] \\
        & \operatorname{xAUC}(o(\operatorname{p}^{a(:i+1)}, \operatorname{p}^{b}) = \operatorname{xAUC}(o(\operatorname{p}^{a(:i)}, \operatorname{p}^{b}) + \mathbb{I} \left[ \operatorname{Y}_{\operatorname{p}^{a(i+1)}}=1\right] \cdot (\sum_{h: h > j}\mathbb{I} \left(\mathrm{Y}_{\operatorname{p}^{b(h)}}=0 \right))
    \end{aligned}
\end{equation}

Obviously, when $i = n^{a}, j = n^{b}$, $\widehat{\operatorname{G}}$ equals to $\operatorname{G}$ in Proposition 2 of the main text.

\subsection{Proposition 2 on PRF metric}
The objective function under pairwise ranking fairness metric (PRF) is as follows:
\begin{equation}
\begin{aligned}
J(o(\operatorname{p}^{a}, \operatorname{p}^{b}))=\text {AUC}(o(\operatorname{p}^{a}, \operatorname{p}^{b}))- \lambda \cdot\Delta\mathrm{PRF}(o(\operatorname{p}^{a}, \operatorname{p}^{b}))
\end{aligned}
\label{eq:optimizationPRF}
\end{equation}

As post-processing procedure does not change iAUC(a) and iAUC(b), the optimization target in Eq.~\ref{eq:optimizationPRF} is equivalent to:

\begin{equation}
\begin{aligned}
\label{eq:PRF_G}
& \text{Maxmizing: } \quad \quad \quad \mathrm{G}(o(\operatorname{p}^{a}, \operatorname{p}^{b}))\\
& \mathrm{G}(o(\operatorname{p}^{a}, \operatorname{p}^{b}))=  k_{a,b} \cdot \mathrm{xAUC}(o(\operatorname{p}^{a}, \operatorname{p}^{b})) + k_{b,a} \cdot \mathrm{xAUC}(o(\operatorname{p}^{a}, \operatorname{p}^{b})) - \lambda \cdot k \cdot \Delta\mathrm{PRF}(o(\operatorname{p}^{a}, \operatorname{p}^{b})).
\end{aligned}
\end{equation}

\section{Proof of Theorem 1}
\subsection{xOrder can achieve the global optimal solution of maximizing Eq.(6) in main text with $\lambda=0$}
\label{Proof: 3_1}
\begin{proof}
We will decompose the problem that maximizing AUC into $(n^{a}+1) \cdot (n^{b}+1)$ subproblems, and $\operatorname{\widehat{G}}$ in the main text in Eq. (9) when $\lambda=0$ is equivalent to the objective in Eq.(\ref{eq:subproblems}). We will use mathematical induction to prove the conclusion that \emph{xOrder can achieve the global optimal solution to maximize $\mathrm{AUC}(o(\operatorname{p}^{a}, \operatorname{p}^{b}))$}. For each subproblem given $i,j$ with $0 \leq i \leq n^a, 0 \leq j \leq n^b$, the optimization target to maximize is:

\begin{equation}
\begin{aligned}
\begin{split}
\widehat{\operatorname{G}}(o(\operatorname{p}^{a(:i)}, \operatorname{p}^{b(:j)})) = \sum\nolimits_{k:k\leqslant i,\mathrm{Y}_{\operatorname{p}^{a(k)}}=1}\sum\nolimits_{h:h\leqslant n^{b},\mathrm{Y}_{\operatorname{p}^{b(h)}}=0}\mathbb{I} \left[\operatorname{p}^{a(k)} \succ \operatorname{p}^{b(h)}\right] +\\
\sum\nolimits_{k:k\leqslant n^a,\mathrm{Y}_{\operatorname{p}^{a(k)}}=0}\sum\nolimits_{h:h\leqslant j,\mathrm{Y}_{\operatorname{p}^{b(h)}}=1}\mathbb{I} \left[\operatorname{p}^{b(h)} \succ \operatorname{p}^{a(k)}\right]
\end{split}
\end{aligned}
\label{eq:subproblems}
\end{equation}

For any $0 \leq i \leq n_a, 0 \leq j \leq n_b$, according to the property in Eq. ~\ref{eq:property}, the update equation when $\operatorname{p}^{a(:i+1)}$ is appended to $o(\operatorname{p}^{a(:i)}, \operatorname{p}^{b(:j)})$:

\begin{equation}
\begin{aligned}
& \widehat{\operatorname{G}}(o(\operatorname{p}^{a(:i)}, \operatorname{p}^{b(:j)}) \textcircled{+} \operatorname{p}^{a(:i+1)} ) = \widehat{\operatorname{G}}(o(\operatorname{p}^{a(:i)}, \operatorname{p}^{b(:j)})) + \mathbb{I} \left[ \operatorname{Y}_{\operatorname{p}^{a(i+1)}}=1\right] \cdot (\sum_{h: h > j}\mathbb{I} \left(\mathrm{Y}_{\operatorname{p}^{b(h)}}=0 \right)).
\end{aligned}
\label{eq:property}
\end{equation}

Consider two trivial cases of $o(\operatorname{p}^{a(:0)}, \operatorname{p}^{b(:j)})$ ($0<j \leq n^{b}$) and $o(\operatorname{p}^{a(:i)}, \operatorname{p}^{b(:0)})$ ($0 < i \leq n^{a}$), there is only one possible path. The unique solution is obtained at the initialized stage of \texttt{xOrder} algorithm.

For any given $i$ and $j$ satisfying $i > 0$ and $j > 0$, suppose \texttt{xOrder} has got the optimal solutions $o^{*}(\operatorname{p}^{a(:i-1)}, \operatorname{p}^{b(:j)})$ and $o^{*}(\operatorname{p}^{a(:i)}, \operatorname{p}^{b(:j-1)}))$ of the subproblems to maximize $\widehat{\operatorname{G}}(o(\operatorname{p}^{a(:i-1)}, \operatorname{p}^{b(:j)}))$ and  $\widehat{\operatorname{G}}(o(\operatorname{p}^{a(:i)}, \operatorname{p}^{b(:j-1)}))$ respectively.

Now, we will prove the solution $o^{*}(\operatorname{p}^{a(:i)}, \operatorname{p}^{b(:j)})$ returned by \texttt{xOrder} is optimal by contradiction. Suppose there exists $\overline{o}(\operatorname{p}^{a(:i)}, \operatorname{p}^{b(:j)})$ satisfying $\widehat{\operatorname{G}}(\overline{o}(\operatorname{p}^{a(:i)}, \operatorname{p}^{b(:j)})) > \widehat{\operatorname{G}}(o^{*}(\operatorname{p}^{a(:i)}, \operatorname{p}^{b(:j)}))$. There are two possible situations: $\overline{o}(\operatorname{p}^{a(:i)}, \operatorname{p}^{b(:j)})$ ends with 1).$\operatorname{p}^{a(i)}$, 2).$\operatorname{p}^{b(j)}$. Without the loss of generality, we assume $\overline{o}(\operatorname{p}^{a(:i)}, \operatorname{p}^{b(:j)})$ ends with $\operatorname{p}^{a(i)}$ and define $\overline{o}(\operatorname{p}^{a(:i-1)}, \operatorname{p}^{b(:j)})$ by $\overline{o}(\operatorname{p}^{a(:i)}, \operatorname{p}^{b(:j)}) = \overline{o}(\operatorname{p}^{a(:i-1)}, \operatorname{p}^{b(:j)}) \textcircled{+} \operatorname{p}^{a(:i)}$. We can get the following inequation:

\begin{equation}
\begin{aligned}
& \widehat{\operatorname{G}}(\overline{o}(\operatorname{p}^{a(:i-1)}, \operatorname{p}^{b(:j)})) = \widehat{\operatorname{G}}(\overline{o}(\operatorname{p}^{a(:i)}, \operatorname{p}^{b(:j)})) - \mathbb{I} \left[ \operatorname{Y}_{\operatorname{p}^{a(i)}}=1\right] \cdot (\sum_{h: h > j}\mathbb{I} \left(\mathrm{Y}_{\operatorname{p}^{b(h)}}=0 \right)) \\
& > \widehat{\operatorname{G}}(o^{*}(\operatorname{p}^{a(:i)}, \operatorname{p}^{b(:j)})) - \mathbb{I} \left[ \operatorname{Y}_{\operatorname{p}^{a(i)}}=1\right] \cdot (\sum_{h: h > j}\mathbb{I} \left(\mathrm{Y}_{\operatorname{p}^{b(h)}}=0 \right))\\
& \geq \widehat{\operatorname{G}}(o^*(\operatorname{p}^{a(:i-1)}, \operatorname{p}^{b(:j)})\textcircled{+} \operatorname{p}^{a(i)}) - \mathbb{I} \left[ \operatorname{Y}_{\operatorname{p}^{a(i)}}=1\right] \cdot (\sum_{h: h > j}\mathbb{I} \left(\mathrm{Y}_{\operatorname{p}^{b(h)}}=0 \right)),\\
& = \widehat{\operatorname{G}}({o}^{*}(\operatorname{p}^{a(:i-1)}, \operatorname{p}^{b(:j)}))
\end{aligned}
\end{equation}

\noindent where $\geq$ in the third line holds due to the update process of \texttt{xOrder} in the main text Eq. (8). This violates the assumption that $o^{*}(\operatorname{p}^{a(:i-1)}, \operatorname{p}^{b(:j)}))$ is opitmal. For the situation that $\overline{o}(\operatorname{p}^{a(:i)}, \operatorname{p}^{b(:j)})$ ends with $\operatorname{p}^{b(j)}$, similarly we can derive $\widehat{\operatorname{G}}(\overline{o}(\operatorname{p}^{a(:i)}, \operatorname{p}^{b(:j-1)})) > \widehat{\operatorname{G}}({o}^{*}(\operatorname{p}^{a(:i)}, \operatorname{p}^{b(:j-1)}))$, which violates the assumption $o^{*}(\operatorname{p}^{a(:i)}, \operatorname{p}^{b(:j-1)}))$ is optimal. Summarizing the deduction above, we can get $o^{*}(\operatorname{p}^{a(:i)}, \operatorname{p}^{b(:j)})$ must be optimal.

Due to the generality of $i$ and $j$, $o^{*}(\operatorname{p}^{a}, \operatorname{p}^{b})$ with $i = n^a$ and $j = n^b$ is the global optimal solution to maximize  $\mathrm{AUC}(o(\operatorname{p}^{a}, \operatorname{p}^{b}))$.
\end{proof}

\subsection{xOrder has the upper bounds of fairness disparities stated in Eq.(10) in main text as $\lambda$ approaches infinity}

In the beginning, we will prove that \texttt{xorder} has the disparity upper bound $\Delta \mathrm{x} \mathrm{AUC} \leq \max \left(\frac{1}{n_{1}^{a}}, \frac{1}{n_{1}^{b}}\right)$. We also decompose the problem in Eq.(6) into $(n^a+1) \cdot (n^b+1)$ subproblems as the last proof does. As $\lambda$ approaches infinity, maxmizing the objective in Eq.(9) means minimizing $\Delta \mathrm{x} \mathrm{AUC}$ exactly and is equivalent to maximizing the objective in Eq.(\ref{eq:lambdainfinity}).

\begin{equation}
\begin{aligned}
&\widehat{\operatorname{G}}(o(\operatorname{p}^{a(:i-1)}, \operatorname{p}^{b(:j)})) = - \left|\widehat{\operatorname{H}}(o(\operatorname{p}^{a(:i-1)}, \operatorname{p}^{b(:j)}))  \right|\\
&\widehat{\operatorname{H}}(o(\operatorname{p}^{a(:i-1)}, \operatorname{p}^{b(:j)}))  = \sum\nolimits_{k:k\leqslant i,\mathrm{Y}_{\operatorname{p}^{a(k)}}=1}\sum\nolimits_{h:h\leqslant n^{b},\mathrm{Y}_{\operatorname{p}^{b(h)}}=0}\mathbb{I} \left[\operatorname{p}^{a(k)} \succ \operatorname{p}^{b(h)}\right] - \\
&\sum\nolimits_{k:k\leqslant n^a,\mathrm{Y}_{\operatorname{p}^{a(k)}}=0}\sum\nolimits_{h:h\leqslant j,\mathrm{Y}_{\operatorname{p}^{b(h)}}=1}\mathbb{I} \left[\operatorname{p}^{b(h)} \succ \operatorname{p}^{a(k)}\right]
\label{eq:lambdainfinity}
\end{aligned}
\end{equation}

\begin{definition}[greedy forward search algorithm] For any $0 \leq i \leq n^a$, $0 \leq j \leq n^b$, given $\overline{o}(\operatorname{p}^{a(:i)}, \operatorname{p}^{b(:j)})$, the update function of greedy forward search algorithm is as follows:

\begin{equation}
\label{eq:greedy}
\begin{aligned}
& \text{Given} \quad o^{*}(\operatorname{p}^{a(:i)}, \operatorname{p}^{b(:j)})\\
& \text{if:}\ \widehat{\operatorname{G}}(o^*(\operatorname{p}^{a(:i)}, \operatorname{p}^{b(:j)})\textcircled{+} \operatorname{p}^{a(i+1)}) \geq \widehat{\operatorname{G}}(o^*(\operatorname{p}^{a(:i)}, \operatorname{p}^{b(:j)})\textcircled{+} \operatorname{p}^{b(j+1)})\\
& \quad \quad \text{update: } \quad o^{*}(\operatorname{p}^{a(:i+1)}, \operatorname{p}^{b(:j)}) = o^{*}(\operatorname{p}^{a(:i)}, \operatorname{p}^{b(:j)})\textcircled{+} \operatorname{p}^{a(i+1)};\\
& \text{otherwise:}\\
& \quad \quad \text{update: } \quad o^{*}(\operatorname{p}^{a(:i)}, \operatorname{p}^{b(:j+1)}) = o^{*}(\operatorname{p}^{a(:i)}, \operatorname{p}^{b(:j)}) \textcircled{+} \operatorname{p}^{b(j+1)},
\end{aligned}
\end{equation}
\end{definition}
Using the algorithm defined above, we can find an optimal path to minimize the xAUC diparity in a greedy manner.

\begin{lemma}
greedy search forward algorithm can achieve the upper bound $\Delta \mathrm{x} \mathrm{AUC} \leq \max \left(\frac{1}{n_{1}^{a}}, \frac{1}{n_{1}^{b}}\right)$.
\label{lm:greedyproof}
\end{lemma}
\begin{proof}
We will use mathematical induction to prove the lemma. As $i = 0$ and $j = 0$, without loss of generality, we assume $\widehat{\operatorname{G}}(o^*(\operatorname{p}^{a(:0)}, \operatorname{p}^{b(:0)})\textcircled{+} \operatorname{p}^{a(1)}) \geq \widehat{\operatorname{G}}(o^*(\operatorname{p}^{a(:0)}, \operatorname{p}^{b(:0)})\textcircled{+} \operatorname{p}^{b(1)})$, then following the update function in Eq.(\ref{eq:greedy}) we get $\widehat{\operatorname{G}}(o^*(\operatorname{p}^{a(:i)}, \operatorname{p}^{b(:j)}) \geq -\min(\frac{1}{n^a_1}, \frac{1}{n^b_1}) \geq -\max(\frac{1}{n^a_1}, \frac{1}{n^b_1})$. Without loss of generality, suppose there exists $\widehat{\operatorname{H}}(o^*(\operatorname{p}^{a(:i)}, \operatorname{p}^{b(:j)}) \leq \max(\frac{1}{n^a_1}, \frac{1}{n^b_1})$ while $\widehat{\operatorname{H}}(o^*(\operatorname{p}^{a(:i+1)}, \operatorname{p}^{b(:j)}) = \widehat{\operatorname{H}}(o^*(\operatorname{p}^{a(:i)}, \operatorname{p}^{b(:j)}) + \mathbb{I}\left[\mathrm{Y}_{\mathrm{p}^{a(i+1)}}=1\right] \cdot\left(\sum_{h: h \leq n^b} \mathbb{I}\left(\mathrm{Y}_{\mathrm{p}^{b(h)}}=0\right)\right) > \max(\frac{1}{n^a_1}, \frac{1}{n^b_1})$ after the next update. As $\operatorname{p}^{a(i+1)})$ is appended to $o^*(\operatorname{p}^{a(:i)}, \operatorname{p}^{b(:j)})$ following the update function Eq.(\ref{eq:greedy}), we can conclude that $\widehat{\operatorname{H}}(o^*(\operatorname{p}^{a(:i)}, \operatorname{p}^{b(:j)})>0$ and
\begin{equation}
\label{eq:lemma1}
\max(\frac{1}{n^a_1}, \frac{1}{n^b_1}) < \widehat{\operatorname{H}}(o^*(\operatorname{p}^{a(:i+1)}, \operatorname{p}^{b(:j)})) \leq -\widehat{\operatorname{G}}(o^*(\operatorname{p}^{a(:i)}, \operatorname{p}^{b(:j+1)})).
\end{equation}
As $0 \leq \widehat{\operatorname{H}}(o^*(\operatorname{p}^{a(:i)}, \operatorname{p}^{b(:j)})) \leq \max(\frac{1}{n^a_1}, \frac{1}{n^b_1})$, $0 \leq \mathbb{I} \left[ \operatorname{Y}_{\operatorname{p}^{b(j+1)}}=1\right] \cdot (\sum_{h: h > i}\mathbb{I} \left(\mathrm{Y}_{\operatorname{p}^{a(h)}}=0 \right)) \leq \max(\frac{1}{n^a_1}, \frac{1}{n^b_1})$, $-\widehat{\operatorname{G}}(o^*(\operatorname{p}^{a(:i)}, \operatorname{p}^{b(:j+1)})) = \left| \widehat{\operatorname{H}}(o^*(\operatorname{p}^{a(:i)}, \operatorname{p}^{b(:j+1)}))\right| = \left|\widehat{\operatorname{H}}(o^*(\operatorname{p}^{a(:i)}, \operatorname{p}^{b(:j)})) - \mathbb{I} \left[ \operatorname{Y}_{\operatorname{p}^{b(j+1)}}=1\right] \cdot (\sum_{h: h \leq n^a}\mathbb{I} \left(\mathrm{Y}_{\operatorname{p}^{a(h)}}=0 \right)) \right| \leq \max(\frac{1}{n^a_1}, \frac{1}{n^b_1})$. This violates the conclusion in Eq.(\ref{eq:lemma1}), thus we prove that $-\widehat{\operatorname{G}}(o^*(\operatorname{p}^{a(:n^a)}, \operatorname{p}^{b(:n^b)})) = \Delta \mathrm{x}\mathrm{AUC} \leq \max \left(\frac{1}{n_{1}^{a}}, \frac{1}{n_{1}^{b}}\right)$ using greedy forward update algorithm.
\end{proof}

\begin{lemma}
The upper bound of $\Delta \mathrm{x} \mathrm{AUC}$ achieved by greedy search forward algorithm is no less than \texttt{xorder}.
\label{lm:greedyproof}
\end{lemma}

\begin{proof}
We also use mathematical induction to prove the lemma. We use $_{\#greedy}\widehat{\mathrm{G}}$ and $_{\#xorder}\widehat{\mathrm{G}}$ to represent the value of $\widehat{\mathrm{G}}$ defined in Eq.(\ref{eq:lambdainfinity}) obtained by \emph{greedy search forward algorithm} and \texttt{xorder}, respectively. As $i = 0$ and $j = 0$, there exists $0 = _{\#greedy}\widehat{\mathrm{G}}(o^*(\operatorname{p}^{a(:0)}, \operatorname{p}^{b(:0)})) \leq _{\#xorder}\widehat{\mathrm{G}}(o^*(\operatorname{p}^{a(:0)}, \operatorname{p}^{b(:0)})) = 0$. Suppose  $_{\#greedy}\widehat{\mathrm{G}}(o^*(\operatorname{p}^{a(:i)}, \operatorname{p}^{b(:j)})) \leq _{\#xorder}\widehat{\mathrm{G}}(o^*(\operatorname{p}^{a(:i)}, \operatorname{p}^{b(:j)}))$, without loss of generality, we assume $\operatorname{p}^{a(i+1)}$ is appended to $o^{*}(\operatorname{p}^{a(:i)}, \operatorname{p}^{b(:j)})$ following the update function Eq.(\ref{eq:greedy}). Now we will prove that $_{\#greedy}\widehat{\mathrm{G}}(o^*(\operatorname{p}^{a(:i+1)}, \operatorname{p}^{b(:j)})) \leq _{\#xorder}\widehat{\mathrm{G}}(o^*(\operatorname{p}^{a(:i+1)}, \operatorname{p}^{b(:j)}))$. From Eq.(8) in the main text, we get the following inequation:
\begin{equation}
\begin{aligned}
_{\#xorder}\widehat{\mathrm{G}}(o^*(\operatorname{p}^{a(:i+1)}, \operatorname{p}^{b(:j)})) &= \max \left( _{\#xorder}\widehat{\mathrm{G}}(o^*(\operatorname{p}^{a(:i)}, \operatorname{p}^{b(:j)}) \textcircled{+} \operatorname{p}^{a(i+1)}, _{\#xorder}\widehat{\mathrm{G}}(o^*(\operatorname{p}^{a(:i+1)}, \operatorname{p}^{b(:j-1)}) \textcircled{+} \operatorname{p}^{b(j)})  \right)\\
& \geq \max \left( _{\#greedy}\widehat{\mathrm{G}}(o^*(\operatorname{p}^{a(:i)}, \operatorname{p}^{b(:j)}) \textcircled{+} \operatorname{p}^{a(i+1)}, _{\#xorder}\widehat{\mathrm{G}}(o^*(\operatorname{p}^{a(:i+1)}, \operatorname{p}^{b(:j-1)}) \textcircled{+} \operatorname{p}^{b(j)})  \right)\\
&\geq _{\#greedy}\widehat{\mathrm{G}}(o^*(\operatorname{p}^{a(:i)}, \operatorname{p}^{b(:j)}) \textcircled{+} \operatorname{p}^{a(i+1)}\\
& \geq _{\#greedy}\widehat{\mathrm{G}}(o^*(\operatorname{p}^{a(:i+1)}, \operatorname{p}^{b(:j)}))
\end{aligned}
\end{equation}
where the $\geq$ in the second line holds due to our assumption that $_{\#greedy}\widehat{\mathrm{G}}(o^*(\operatorname{p}^{a(:i)}, \operatorname{p}^{b(:j)})) \leq _{\#xorder}\widehat{\mathrm{G}}(o^*(\operatorname{p}^{a(:i)}, \operatorname{p}^{b(:j)}))$. Summarizing the deduction above, we get that $_{\#greedy}\widehat{\mathrm{G}}(o^*(\operatorname{p}^{a(:n^a)}, \operatorname{p}^{b(:n^b)})) \leq _{\#xorder}\widehat{\mathrm{G}}(o^*(\operatorname{p}^{a(:n^a)}, \operatorname{p}^{b(:n^b)}))$. Recalling that $\widehat{\mathrm{G}}(o^*(\operatorname{p}^{a(:n^a)}, \operatorname{p}^{b(:n^b)})) = -\Delta \mathrm{x} \mathrm{AUC}$, we prove that the upper bound of $\Delta \mathrm{x} \mathrm{AUC}$ achieved by greedy search forward algorithm is no less than \texttt{xorder}.
\end{proof}

Combining the two lemmas above, we prove that \texttt{xorder} achieves that $\Delta \mathrm{x} \mathrm{AUC} \leq \max \left(\frac{1}{n_{1}^{a}}, \frac{1}{n_{1}^{b}}\right)$ as $\lambda$ approaches infinity. For the proposition that \texttt{xorder} achieves that $\Delta \mathrm{PRF} \leq \max \left(\frac{n_{0}^{b}}{n_{0} \cdot n_{1}^{a}}, \frac{n_{0}^{a}}{n_{0} \cdot n_{1}^{b}}\right)$ as $\lambda$ approaches infinity, we can prove it in a similar way.

\section{Implementation details}
\subsection{Data preprocess}
For the four fairness benchmark data sets, we use the preprocessed data sets from the resource of \cite{kallus2019fairness}, while three of them(COMPAS, Adult) are from the resource of \cite{friedler2019comparative}. All categorical variables will be encoded as one-hot features. For Framingham, we use all features as the input. For MIMIC-III data set, we preprocess the original data as in \cite{harutyunyan2019multitask} for each ICU admission. EHR data of 17 selected clinical variables from the first 48 hours are used to extract features. For all clinical variables, different statistics (mean, std, etc) of different time slices are extracted to form a vector with 714 features.

\subsection{Training process of model}

We first train a model without any fairness regularization to obtain the unadjusted result. The linear model is optimized by gradient descent with learning rate of 1.0 and tfco respectively. The model is trained for at most 100 epochs. If the training loss failed to reduce after 5 consecutive epochs, the training will be stopped. For bipartite rankboost model, the number of estimators is 50 and learning rate is 1.0, which are the same as in the experiments of \cite{kallus2019fairness}.

\noindent \textbf{Post-logit} With transformation function $f(x) = \frac{1}{1 + e^{-(\alpha x + \beta)}}$, We follow the same procedure in \cite{kallus2019fairness} to optimize empirical disparity($\Delta \mathrm{xAUC}$ or $\Delta \mathrm{PRF}$) with fixed $\beta = -2$. The value of $\alpha$ is chosen from $\left[0, 10\right]$. However, we find that this setting can not obtain equal $\mathrm{xAUC}$ or $\mathrm{PRF}$ in MIMIC data set which has not been used in \cite{kallus2019fairness}. So we change $\beta$  from $\left[-1, -3\right]$ and $\alpha$ from $\left[0, 10\right]$ when optimizing the empirical disparity on MIMIC data set.

\noindent \textbf{Corr-reg} For corr-reg, we train models of the same structure with various weights of fairness regularization. Since the correlation regularization is only an approximation of the pairwise ranking disparity, we use the corresponding ranking fairness metrics($\Delta \mathrm{xAUC}$ or $\Delta \mathrm{PRF}$) as the criterion to determine the range of weights. We initialize the weight to be 0 (equivalent to unadjusted) and increase it until the average ranking disparity on training data is lower than $0.0001$ or does not decrease in 2 consecutive steps. We also apply this strategy to \texttt{xOrder}.

\sen{\noindent \textbf{Opti-tfco} For opti-tfco, we use the implementation of \cite{narasimhan2020pairwise}. When considering the trade-off between utility and fairness in a bipartite ranking problem, the authors optimize a constrained problem which is defined as follows: the model is trained to maximize AUC under the constrain that disparity($\Delta \mathrm{xAUC}$ or $\Delta \mathrm{PRF}$) is smaller than hyper-parameter $\epsilon$. }

\subsection{Computing infrastructure and consumption}

We run \texttt{xOrder} on computer with Intel i7-9750H CPU(@2.6GHz $\times$ 6) and 16 gb RAM. We report the average running time of 10 evaluation runs in Table \ref{table:time_consumption}.

\begin{table}
    \small
    \caption{Time consumption of \texttt{xOrder}}
    \label{table:time_consumption}
    \centering
    \begin{tabular}{c|c|c|c|c|c}
    \hline
    Data set & COMPAS & Adult & Framingham & MIMIC-III & eICU \\
    \hline
    $n$ & 6,167 & 30,162 & 4,658 & 21,139 & 17,402\\
    time/sec & $3.27 \pm 0.44$ & $38.67 \pm 1.28$ & $2.11 \pm 0.41$ & $16.71 \pm 0.61$ & $12.44 \pm 0.53$\\
    \hline
    \end{tabular}
\end{table}

\begin{table}
    \footnotesize
    \caption{Summary of ranking fairness metrics on unadjusted result over 10 repeat experiments}
    \label{table:summary_unadjusted}
    \centering
    \begin{tabular}{c|c|c|c|c|c|c}
        \hline
        \multirow{2}{*}{Data set}
        & \multicolumn{2}{|c|}{Linear Model(gradient descent)}
        & \multicolumn{2}{|c|}{Linear Model(AUC)}&\multicolumn{2}{|c}{Bipartite Rankboost} \\
        \cline{2-7}
        & $\Delta \mathrm{xAUC}$ & $\Delta \mathrm{PRF}$ & $\Delta \mathrm{xAUC}$ & $\Delta \mathrm{PRF}$ & $\Delta \mathrm{xAUC}$ & $\Delta \mathrm{PRF}$ \\
        \hline
        COMPAS & $0.195 \pm 0.024$ & $0.119 \pm 0.015$ & $0.220 \pm 0.026$ & $0.138 \pm 0.019$ & $0.206 \pm 0.027$ & $0.134 \pm 0.019$ \\
        Adult & $0.089 \pm 0.011$ & $0.040 \pm 0.009$ & $0.086 \pm 0.010$ & $0.034 \pm 0.008$ & $0.055 \pm 0.016$ & $0.020 \pm 0.008$\\
        Framingham & $0.270 \pm 0.022$ & $0.124 \pm 0.018$ & $0.213 \pm 0.089$ & $0.117 \pm 0.067$ & $0.285 \pm 0.075$ & $0.136 \pm 0.038$\\
        MIMIC, mortality-gender & $0.050 \pm 0.012$ & $0.020 \pm 0.009$ & $0.072 \pm 0.018$ & $0.034 \pm 0.012$ & ${0.018 \pm 0.011}^{*}$ & ${0.010 \pm 0.010}^{*}$\\
        MIMIC, prolonged LOS-ethnicity & $0.043 \pm 0.008$ & $0.021 \pm 0.008$ & $0.034 \pm 0.012$ & ${0.010 \pm 0.008}^{*}$ & $0.025 \pm 0.010$ & $0.012 \pm 0.007$\\
        eICU, prolonged LOS-ethnicity & $0.088 \pm 0.018$ & $0.029 \pm 0.015$ & $0.076 \pm 0.017$ & $0.026 \pm 0.012$ & $0.083 \pm 0.029$ & $0.032 \pm 0.018$\\
        \hline
    \end{tabular}
\end{table}

\section{Additional Experiment Results}

\subsection{Ranking fairness analysis on unadjusted result}

We report the ranking fairness metrics($\Delta \mathrm{xAUC}$, $\Delta \mathrm{PRF}$) in Table \ref{table:summary_unadjusted}. Large $\Delta \mathrm{xAUC}$ and $\Delta \mathrm{PRF}$ are observed on the five benchmark data sets. We use t-test with p-value as $0.0001$ to evaluate whether the average $\Delta \mathrm{xAUC}$ and $\Delta \mathrm{PRF}$ do not equal to 0.We mark the results which do not pass the test with $*$ on the table. For MIMIC-III and eICU data set, disparities are significant with certain $\operatorname{Y}-\operatorname{A}$.

\subsection{Complete experiment results}
The complete results with model-metric combinations of liner model-$\Delta \mathrm{xAUC}$ optimized by gradient descent, liner model-$\Delta \mathrm{xAUC}$ optimized by tfco, linear model-$\Delta \mathrm{PRF}$ optimized by gradient descent, linear model-$\Delta \mathrm{PRF}$ optimized by tfco, bipartite rankboost-$\Delta \mathrm{xAUC}$ and bipartite rankboost-$\Delta \mathrm{PRF}$ are shown in Figure \ref{fig:lr_result_xauc_gd}, \ref{fig:lr_result_xauc_tfco}, \ref{fig:lr_result_prf_gd}, \ref{fig:lr_result_prf_tfco}, \ref{fig:rb_result_xauc} and \ref{fig:rb_result_prf} respectively. The source codes to reproduce these results on public data sets are at https://github.com/xOrder-code/xOrder. Most subfigures in Figure \ref{fig:lr_result_xauc_tfco}, \ref{fig:lr_result_prf_tfco} and \ref{fig:rb_result_xauc} have been discussed in the main text.

\begin{figure}[htbp]
    \centering
    \includegraphics[width=5.5in]{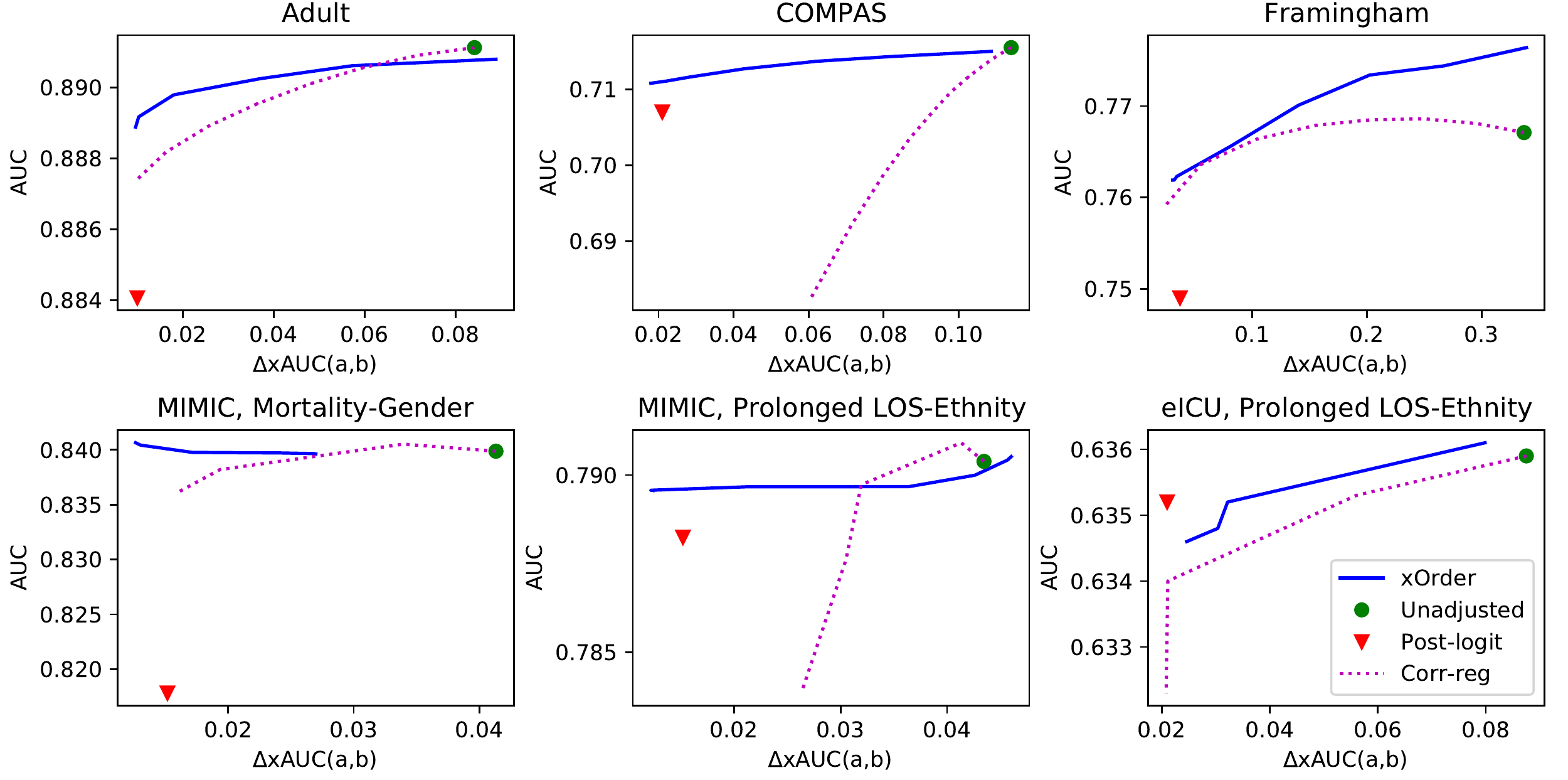}
    \caption{$\mathrm{AUC}$-$\Delta \mathrm{xAUC}$ with linear model trained by gradient descent.}
    \label{fig:lr_result_xauc_gd}
\end{figure}

\begin{figure}[htbp]
    \centering
    \includegraphics[width=5.5in]{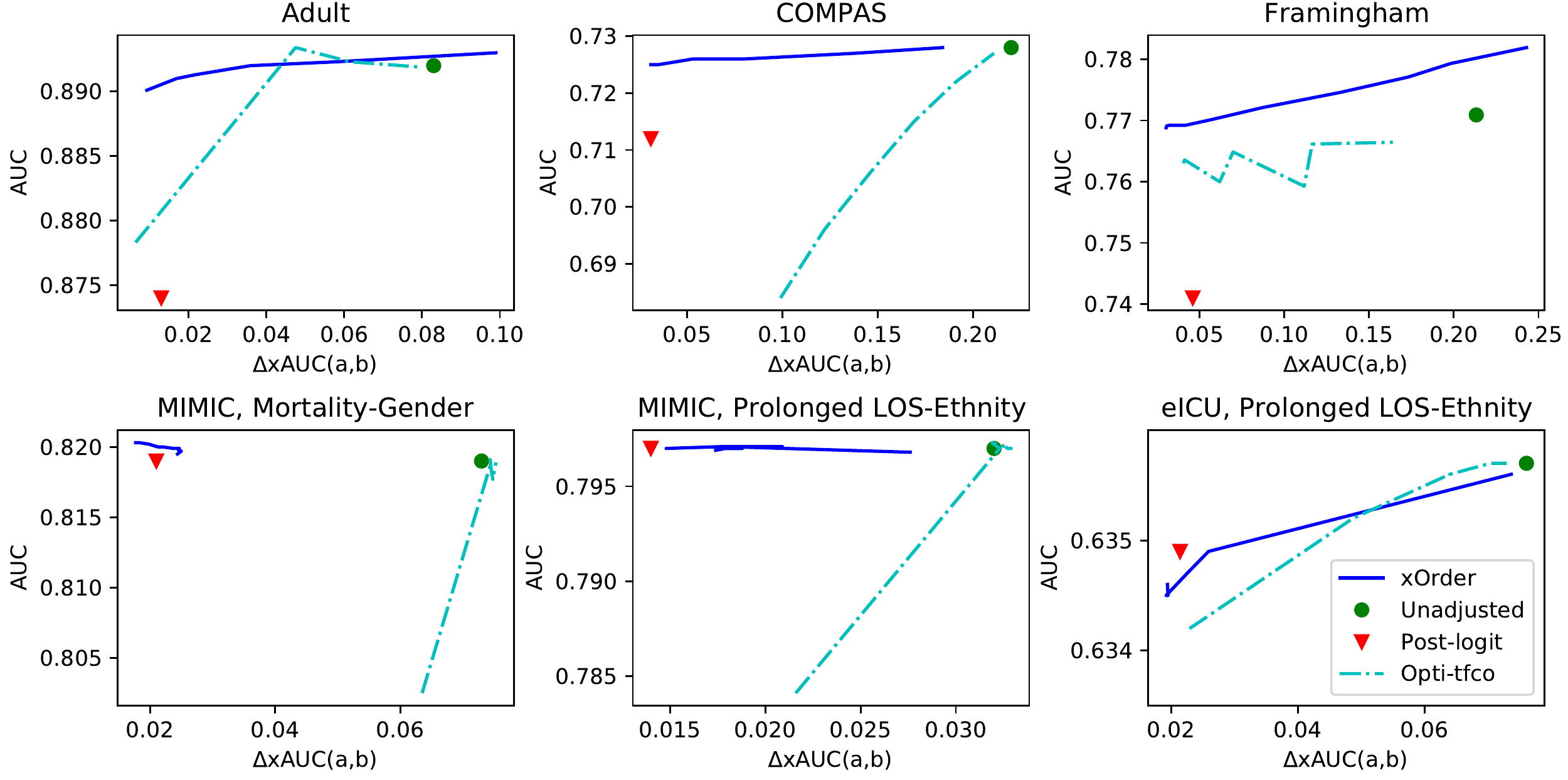}
    \caption{$\mathrm{AUC}$-$\Delta \mathrm{xAUC}$ with linear model trained by tfco.}
    \label{fig:lr_result_xauc_tfco}
\end{figure}

\begin{figure}[htbp]
    \centering
    \includegraphics[width=5.5in]{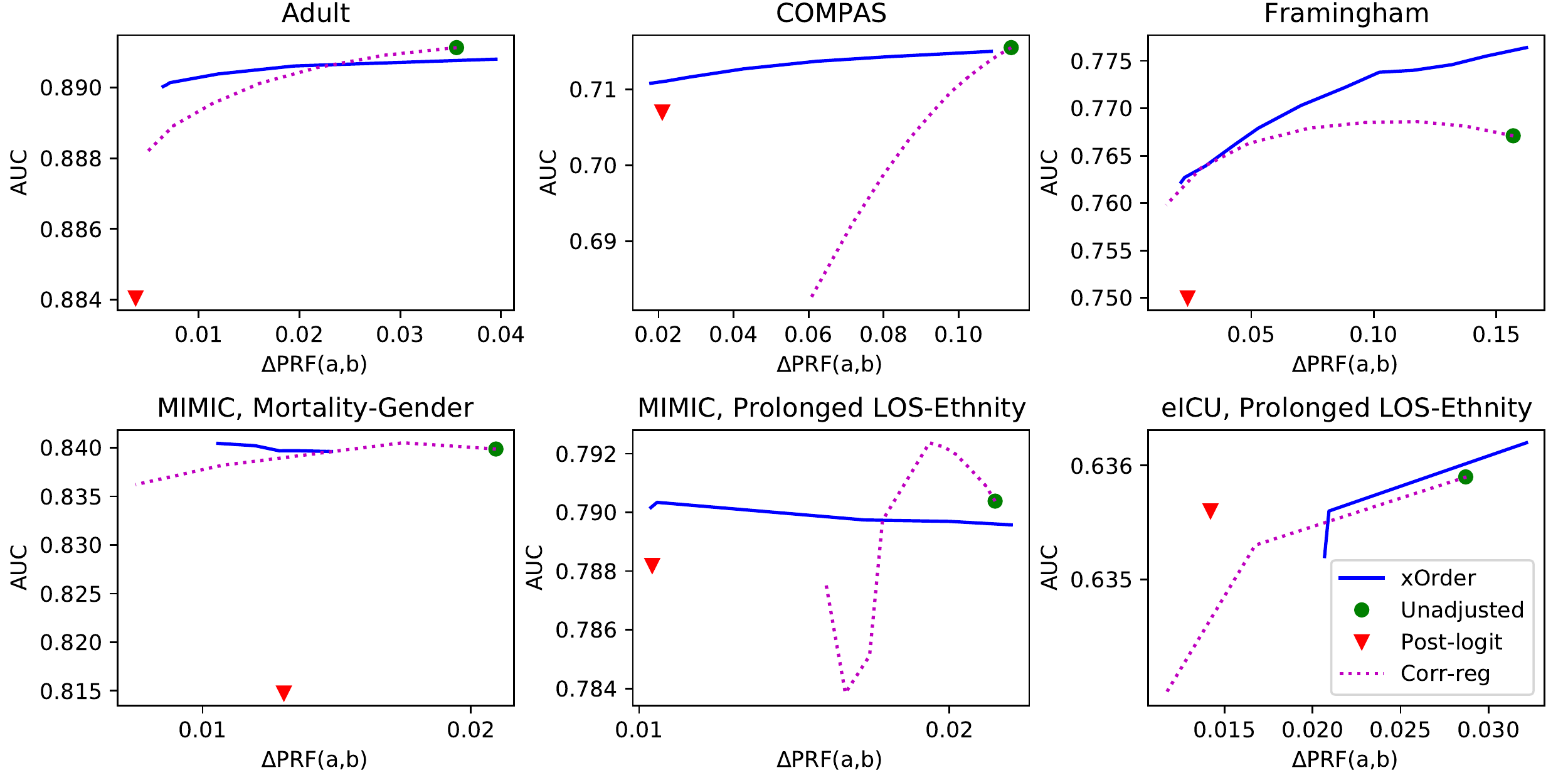}
    \caption{$\mathrm{AUC}$-$\Delta \mathrm{PRF}$ with linear model trained by gradient descent.}
    \label{fig:lr_result_prf_gd}
\end{figure}

\begin{figure}[htbp]
    \centering
    \includegraphics[width=5.5in]{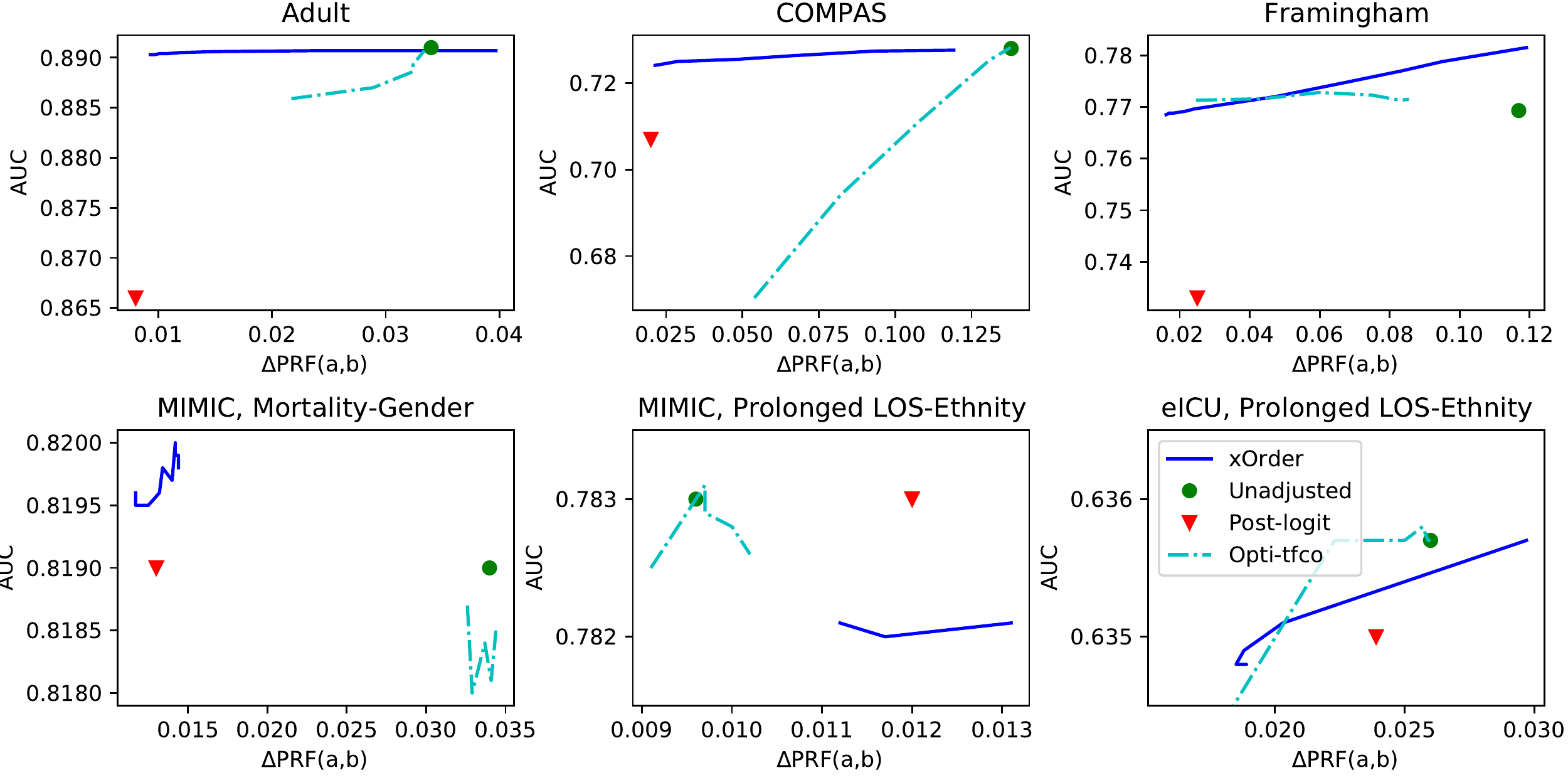}
    \caption{$\mathrm{AUC}$-$\Delta \mathrm{PRF}$ with linear model trained by tfco.}
    \label{fig:lr_result_prf_tfco}
\end{figure}

\begin{figure}[htbp]
    \centering
    \includegraphics[width=5.5in]{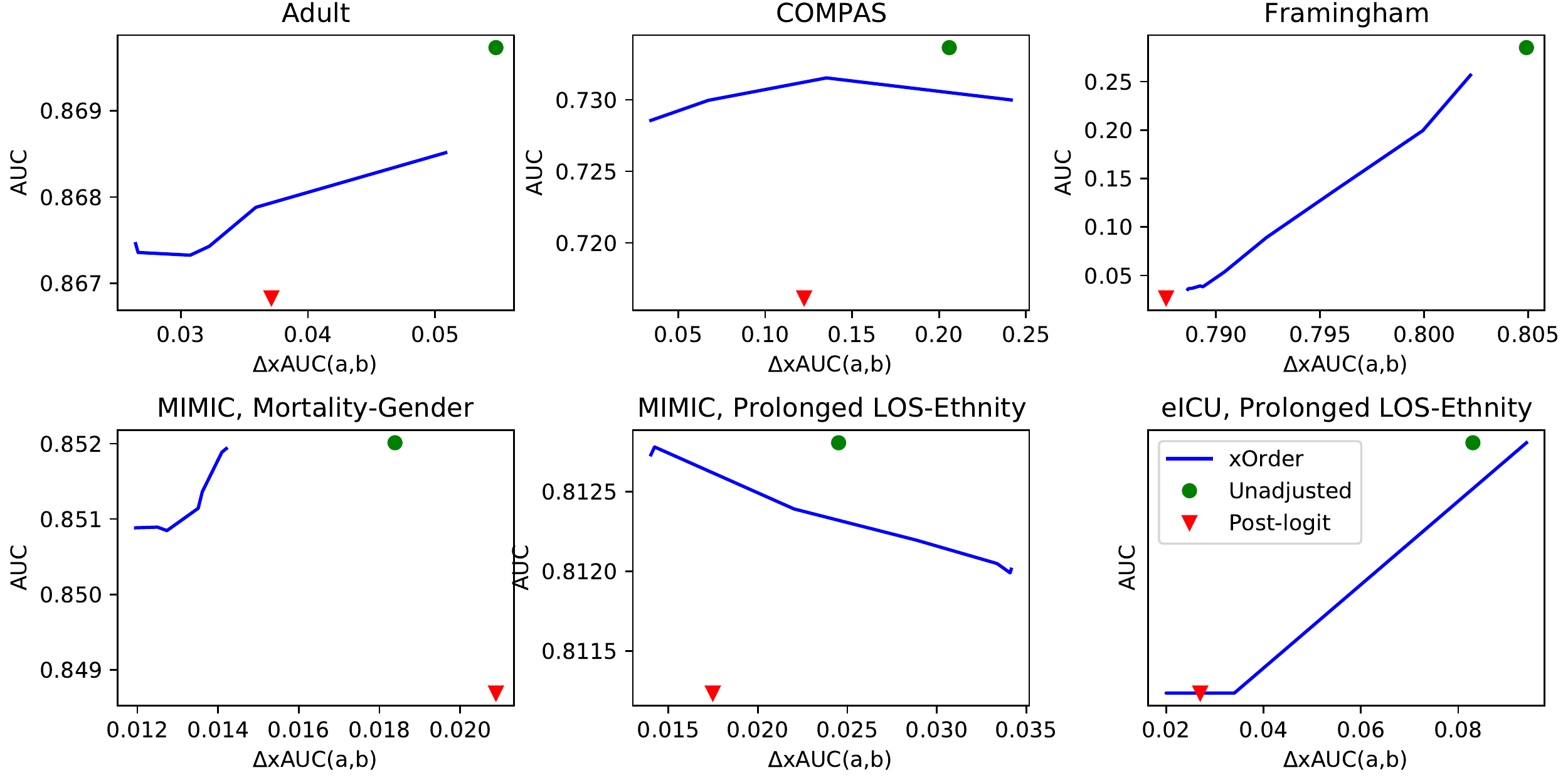}
    \caption{$\mathrm{AUC}$-$\Delta \mathrm{xAUC}$ with bipartite rankboost model.}
    \label{fig:rb_result_xauc}
\end{figure}

\begin{figure}[htbp]
    \centering
    \includegraphics[width=5.5in]{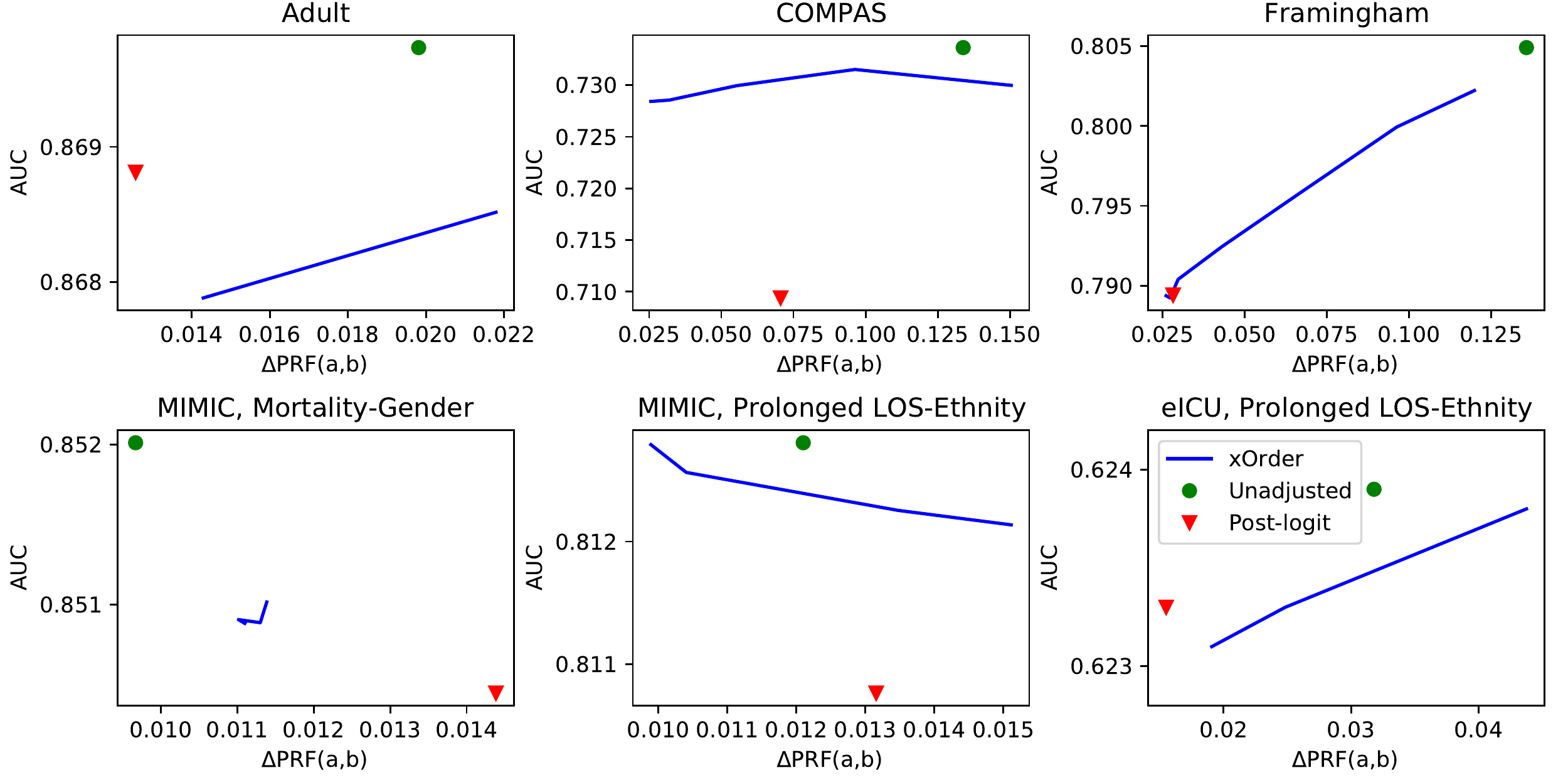}
    \caption{$\mathrm{AUC}$-$\mathrm{\Delta PRF}$ with bipartite rankboost model.}
    \label{fig:rb_result_prf}
\end{figure}

\end{document}